\documentclass[11pt]{article}

\usepackage{amsmath,amssymb,amsthm,mathtools}
\usepackage{bbm}
\usepackage{graphicx}
\usepackage{booktabs}
\usepackage{algorithm}
\usepackage{algpseudocode}
\usepackage{xurl}
\usepackage{arxiv}
\usepackage{subcaption}
\usepackage{float}
\usepackage{dsfont}
\usepackage{enumitem}

\newcommand{\sX}{\mathcal{X}}

\newcommand{\sH}{\mathcal{H}}

\newcommand{\sW}{\mathcal{W}}

\newcommand{\E}{\mathbb{E}}

\DeclareMathOperator*{\argmin}{arg\,min}

\theoremstyle{plain}
\newtheorem{theorem}{Theorem}[section]
\newtheorem{corollary}[theorem]{Corollary}

\newtheorem{lemma}[theorem]{Lemma}

\theoremstyle{definition}

\theoremstyle{remark}

\makeatletter
\AddToHook{env/theorem/after}{\@afterindentfalse\@afterheading}
\AddToHook{env/corollary/after}{\@afterindentfalse\@afterheading}
\AddToHook{env/proposition/after}{\@afterindentfalse\@afterheading}
\AddToHook{env/lemma/after}{\@afterindentfalse\@afterheading}
\AddToHook{env/definition/after}{\@afterindentfalse\@afterheading}
\AddToHook{env/remark/after}{\@afterindentfalse\@afterheading}
\AddToHook{env/proof/after}{\@afterindentfalse\@afterheading}
\AddToHook{env/enumerate/after}{\@afterindentfalse\@afterheading}
\AddToHook{env/itemize/after}{\@afterindentfalse\@afterheading}
\makeatother

\title{Bellman Calibration for Marginalized Importance Weighting in Offline Reinforcement Learning}

\author{
  Lars van der Laan \\
 Stanford University \\
  \texttt{vdlaan@stanford.edu}
  \And
  Nathan Kallus \\
  Netflix and Cornell University
}

\begin{document}

\maketitle

 \begin{abstract}
Marginalized importance weighting evaluates a target policy by reweighting
offline state--action samples with its discounted occupancy ratio, characterized
by an adjoint Bellman equation. Existing minimax, primal--dual, and fitted
fixed-point estimators can leave residual occupancy-balance violations because
of function-class approximation, regularization, or incomplete optimization.
These violations are difficult to diagnose and reduce because the objectives
generally lack a direct supervised validation loss for hyperparameter tuning,
model selection, and early stopping. We introduce isotonic Bellman calibration, a one-dimensional, model-agnostic
post-processing method that reduces these violations while preserving the
ranking information in any initial occupancy-ratio estimate. The method corrects
the estimate's scale and shape by applying fitted occupancy-ratio evaluation
(\textsc{FORE}) over a one-dimensional class of nondecreasing transformations.  We characterize Bellman calibration as a conditional fixed-point property
equivalent to occupancy-balance against every test function of the
calibrated ratio. More generally, we derive a calibration--refinement bound
showing that any fitted ratio with small calibration error performs nearly as
well as the best post-processing based on its fitted values. For isotonic
Bellman calibration, we establish finite-sample calibration guarantees and a KL
oracle inequality relative to the best monotone transformation of the initial
estimate. Consequently, isotonic Bellman calibration achieves small
calibration error and KL risk within statistical error of the best monotone
correction, with guarantees for downstream target-occupancy functionals,
including policy-value estimation.  
\end{abstract}

\section{Introduction}

Marginalized importance weighting evaluates a target policy by reweighting
samples from an offline state--action distribution using the target policy's
discounted occupancy ratio
\citep{liuEtAl2018InfiniteHorizonOPE,xieEtAl2019MarginalizedIS,
yinWang2020EfficientTabularOPE}. Unlike trajectory-level importance sampling,
it corrects marginal state--action occupancies without multiplying importance
ratios across time
\citep{jiangLi2016DoublyRobust,liuEtAl2018InfiniteHorizonOPE,
xieEtAl2019MarginalizedIS}. Because the occupancy ratio depends on the target
policy and transition law, but not on the reward function, a single estimate
can evaluate multiple rewards without fitting a separate value function for
each. It can also be combined with a fitted value or (Q)-function to obtain
doubly robust and semiparametrically efficient estimators
\citep{kallusUehara2020DoubleRL,kallusUehara2022BreakingHorizonDRL,
ueharaEtAl2020MWLMQL}. Accurate occupancy-ratio estimation is therefore a
central problem in offline policy evaluation.

The discounted occupancy ratio satisfies an adjoint Bellman equation, which
induces balance identities linking the reweighted offline distribution to the
target policy's initial distribution and one-step dynamics
\citep{nachumEtAl2019DualDICE,ueharaEtAl2021FiniteSampleMinimax}. Many
estimators enforce these identities through saddle-point problems over weight
and critic classes, choosing weights to minimize the worst-case empirical
violation of the resulting Bellman moments. This principle underlies minimax
weighting and primal--dual methods such as \textsc{DualDICE}
\citep{nachumEtAl2019DualDICE,ueharaEtAl2020MWLMQL,
ueharaEtAl2021FiniteSampleMinimax}. Fitted occupancy-ratio evaluation
(\textsc{FORE}) instead targets the adjoint Bellman fixed point directly through
a sequence of ratio-fitting problems analogous to fitted (Q)-evaluation
\citep{van2026fitted}.

Despite these advances, marginalized importance weighting can remain
statistically difficult and unstable. Under limited overlap, the true occupancy
ratio may be large in regions poorly represented in the offline data, making
both the ratio and the resulting value estimates highly variable.
Occupancy-ratio estimation is further complicated by its fixed-point
characterization, which provides neither observed labels nor a canonical
held-out loss for tuning and model selection. In minimax, primal--dual, and
fitted fixed-point methods, regularization, restricted ratio or critic classes,
optimization error, and early termination can also leave substantial
population Bellman imbalance even when the empirical objective is small
\citep{nachumEtAl2019DualDICE,ueharaEtAl2020MWLMQL,
ueharaEtAl2021FiniteSampleMinimax}. Yet an imperfect ratio estimate may still
reliably rank regions from relatively low to relatively high target occupancy,
even when its magnitudes are poorly calibrated for Bellman balance.

This observation motivates post-processing methods that stabilize fitted
ratios while preserving their ordering information. Standard approaches
include normalization, trimming, and clipping, which rescale weights, discard
observations with extreme fitted ratios, or cap extreme values
\citep{robins2000marginal,crump2009dealing,zhou2020propensity,
lee2011weight,ionides2008truncated,munos2016safe,espeholt2018impala,
zhang2022truncated}. These transformations reduce weight dispersion by
targeting prespecified distortions, such as global scale or tail behavior, and
often require tuning parameters whose bias--variance tradeoff is
difficult to assess from observed data
\citep{lee2011weight,ionides2008truncated,munos2016safe,
espeholt2018impala,su2020doubly,orenstein2022robust,
van2025stabilized}.

To directly correct residual imbalance, we adapt post-hoc calibration from
supervised learning
\citep{lichtenstein1977calibration,bella2010calibration,guo2017calibration,
gupta2021distribution,noarov2023scope}. In classification and regression, a
predictor is calibrated if the average outcome among observations receiving a
given prediction equals that prediction. For example, an event assigned
probability $0.6$ should occur approximately $60\%$ of the time.
Analogously, we call a fitted occupancy ratio calibrated if, conditional on its
fitted value, it agrees on average with one application of the adjoint Bellman
operator. This is natural because the true occupancy ratio is a fixed point of
that operator. Equivalently, the adjoint Bellman moments are balanced against
every function of the fitted ratio. Calibration therefore corrects the
magnitudes relevant for Bellman balance while preserving the ordering encoded
by the initial estimate.

We develop a simple post-processing method that applies to any initial
occupancy-ratio estimator and learns a nonparametric one-dimensional monotone
correction. The method
provides finite-sample nonparametric Bellman calibration guarantees. Our
method, \emph{isotonic Bellman calibration}, applies fitted occupancy-ratio
evaluation (\textsc{FORE}) over a one-dimensional class of nondecreasing
transformations, using adjoint Bellman targets in place of observed calibration
labels. It extends classical isotonic calibration
\citep{zadrozny2002transforming,niculescu2005predicting}, which uses isotonic
regression \citep{barlow1972isotonic} to learn a monotone transformation of an
initial predictor. Monotonicity preserves the ordering of the initial estimate,
while the adjoint Bellman equation determines the scale and shape of the
correction. The transformation class contains rescaling, clipping, and other
monotone parametric adjustments, while allowing more flexible, data-adaptive
corrections across the range of fitted ratios. The procedure operates as a
low-dimensional post-processing step for an already fitted ratio model.

\medskip
\noindent\textbf{Contributions.}
Our contributions are threefold.
\begin{enumerate}[leftmargin=1.5em,topsep=0.2em,itemsep=0.25em,
parsep=0pt,partopsep=0pt]

\item We formalize \textbf{Bellman calibration for occupancy ratios} through
adjoint Bellman moments, characterize it as a conditional fixed-point property,
and derive a calibration--refinement bound that controls occupancy-ratio error
by the calibration error and the best approximation to the true ratio by a
transformation of the fitted ratio.

\item We develop \textbf{fitted isotonic Bellman calibration}, a model-agnostic
post-processing method based on fitted occupancy-ratio evaluation
(\textsc{FORE}). Its core form requires only a stopping rule. Each iteration
solves a convex isotonic optimization problem whose optimality conditions imply
an exact empirical Bellman-balance identity.

\item We establish \textbf{finite-sample calibration and accuracy guarantees},
including a nonasymptotic bound on Bellman calibration error under
subexponential initial coverage and one-step smoothing, and a KL regret bound showing that, up to statistical error, calibration preserves and may
improve occupancy-ratio estimation accuracy.

\end{enumerate}

\subsection{Related work}

\paragraph{Off-policy evaluation and occupancy corrections.}
Classical off-policy evaluation uses trajectory-level or per-decision
importance ratios, while doubly robust estimators combine importance weighting
with value-function estimates
\citep{thomasBrunskill2016DataEfficientOPE,jiangLi2016DoublyRobust}.
Marginalized importance sampling avoids products of trajectory ratios by
correcting marginal state or state--action occupancies
\citep{xieEtAl2019MarginalizedIS,yinWang2020EfficientTabularOPE,
liuEtAl2018InfiniteHorizonOPE}. Semiparametric theory likewise identifies the
occupancy ratio and value function as the two nuisance functions underlying
efficient and doubly robust off-policy evaluation
\citep{kallusUehara2019IntrinsicOPE,kallusUehara2020DRL_ICML,
kallusUehara2020DoubleRL,ueharaShiKallus2022OPEReview,
kallusUehara2022BreakingHorizonDRL,vanDerLaanEtAl2025AutomaticDRL,
van2025efficient}. Related stationary-distribution and occupancy corrections
are also used to mitigate distribution shift in off-policy temporal-difference
learning and fitted value-function methods
\citep{suttonEtAl2016EmphaticTD,hallakMannor2017COPTD,
geladaBellemare2019CovariateShift,pattersonEtAl2022GeneralizedPBE,
vanDerLaanKallus2025StationaryWeightedFQE,
vanDerLaanKallus2025SoftFQI}. Under weak overlap, recent work also truncates
estimated occupancy ratios to stabilize doubly robust off-policy estimators
\citep{mehrabiWager2024WeakOverlap}. Such methods rely on prespecified or
theoretically guided truncation levels whose bias--variance tradeoff can be
difficult to assess from the observed data.

\paragraph{Primal--dual, minimax, and fitted ratio estimation.}
Many occupancy-ratio estimators enforce balance or stationarity restrictions
through saddle-point, minimax, or temporal-difference objectives.
\textsc{DualDICE} estimates discounted distribution corrections without
behavior-policy probabilities or products of trajectory ratios
\citep{nachumEtAl2019DualDICE}, while \textsc{GenDICE} and
\textsc{GradientDICE} extend this approach to stationary-distribution
correction
\citep{zhangEtAl2020GenDICE,zhangEtAl2020GradientDICE}. Related work develops
minimax weight and value-function learners, regularized Lagrangian
formulations, confidence procedures, and regression-based variants
\citep{liuEtAl2018InfiniteHorizonOPE,ueharaEtAl2020MWLMQL,
ueharaEtAl2021FiniteSampleMinimax,yangEtAl2020RegularizedLagrangianDICE,
daiEtAl2020CoinDICE,cheEtAl2025AVGDICE}. Related finite-horizon work includes
the FORC estimator of \citet{huangChenJiang2023DensityFeatures}; see also
\citet{huangJiang2024OccupancyPG}. FORC recursively fits stagewise occupancy
ratios by squared-loss regression under stagewise realizability conditions, a
density-ratio analogue of Bellman completeness. In the discounted setting,
fitted occupancy-ratio evaluation (\textsc{FORE}) instead recursively applies
the adjoint Bellman map and projects each image through a single-level
density-ratio objective \citep{van2026fitted}. Its guarantees require
realizability or approximation of the target occupancy ratio, rather than
richness of a separate critic class or adjoint Bellman completeness.

We build on \textsc{FORE} by repurposing its fitted recursion for post-hoc
calibration. Rather than fitting another high-dimensional ratio model, our
method restricts each iteration to a nondecreasing transformation of an
arbitrary initial estimate. This yields a model-agnostic calibration procedure
with no regularization or transformation-shape hyperparameter in its core form.
It provides exact empirical Bellman-balance identities together with
finite-sample guarantees for calibration and occupancy-ratio accuracy.

\paragraph{Calibration.}
Post-hoc calibration transforms an initial predictor to satisfy
self-consistency within strata of its fitted values. In supervised learning,
common methods include Platt scaling, histogram binning, and isotonic regression
\citep{platt1999probabilistic,zadrozny2001obtaining,zadrozny2002transforming,
niculescu2005predicting,guo2017calibration,
kuleshov2015calibrated,kuleshov2018accurate}. General theory for isotonic
calibration in regression and loss-based prediction is developed by
\citet{van2025generalized}. In causal inference, calibration methods have been
used to recalibrate propensity scores and inverse-probability weights, improving
the finite-sample performance of estimators of average treatment effects and
single-stage policy values
\citep{gutman2022propensity,deshpande2023calibrated,
ballinari2024improvingfinitesampleperformance,van2023causal,
van2025stabilized,van2024doubly,
rabenseifner2025calibrationstrategiesrobustcausal}. These methods concern
one-step importance weights; discounted occupancy ratios instead satisfy a
recursive adjoint Bellman equation. In the nondynamic special case
\(\gamma=0\), our method and guarantees recover
and extend the isotonic inverse-probability-weight calibration results of
\citet{van2025stabilized} from static treatment interventions to general
treatment policies. 

The closest related work in reinforcement learning develops Bellman calibration
for value prediction, using fitted \(Q\)-evaluation and histogram binning to
recalibrate value functions against one-step Bellman targets
\citep{van2025bellman}. We develop the adjoint analogue of Bellman calibration
for occupancy ratios. Specifically, we replace the Bellman operator with its
adjoint, use \textsc{FORE} rather than fitted \(Q\)-evaluation, and introduce
an isotonic procedure with finite-sample Bellman-balance and occupancy-ratio
accuracy guarantees. The proof of our calibration guarantees in
Section~\ref{sec:occupancy-calibration-error} builds on calibration arguments
developed for supervised losses in
\cite{van2023causal,whitehouse2024orthogonal,van2025generalized}.

\section{Discounted occupancy ratios and Bellman balance}

\subsection{Discounted MDPs and occupancy ratios}

We consider a discounted MDP
\((\mathcal S,\mathcal A,P,\gamma)\), where \(\mathcal S\) and
\(\mathcal A\) are measurable state and action spaces,
\(P(\cdot\mid s,a)\) is the state-transition kernel, and
\(\gamma\in[0,1)\) is the discount factor. Let
\(X=(S,A)\in\mathcal X:=\mathcal S\times\mathcal A\), and let \(\nu\)
denote the offline state--action sampling distribution. For a stationary
ergodic trajectory generated by a behavior policy, \(\nu\) may be taken as
the stationary distribution of the induced state--action process. When
multiple finite trajectories are pooled, \(\nu\) may instead be taken as the
time-averaged state--action distribution over the sampled time points.

For a target policy \(\pi\) and initial state distribution \(\rho_0\), define
the initial state--action distribution and the policy-induced transition
kernel on state--action pairs by
\[
  \mu_{0,\pi}(ds,da)=\rho_0(ds)\pi(da\mid s),
  \qquad
  P_\pi(dx'\mid x)=P(ds'\mid s,a)\pi(da'\mid s'),
\]
where \(x=(s,a)\) and \(x'=(s',a')\). For any integrable
\(g:\mathcal X\to\mathbb R\), write
\((P_\pi g)(x):=\int g(x')P_\pi(dx'\mid x)\).

The normalized discounted occupancy measure under \(\pi\) is
\[
  \mu_\pi
  =
  (1-\gamma)\sum_{t\geq 0}\gamma^t
  \mathcal L_\pi(X_t),
\]
where \(X_0\sim\mu_{0,\pi}\) and
\(X_{t+1}\mid X_t\sim P_\pi(\cdot\mid X_t)\). Assuming \(\mu_\pi\ll\nu\),
define the discounted occupancy ratio by
\(w_\pi(x):=(d\mu_\pi/d\nu)(x)\).
This ratio underlies marginalized importance weighting
\citep{liuEtAl2018InfiniteHorizonOPE,ueharaEtAl2020MWLMQL}. In particular, for
any integrable reward function \(r:\mathcal X\to\mathbb R\), the normalized
target-policy value satisfies
\[
  V_\pi(r)
  :=
  (1-\gamma)\mathbb E_\pi
  \left[
    \sum_{t\geq 0}\gamma^t r(S_t,A_t)
  \right]
  =
  \mathbb E_{\mu_\pi}\{r(X)\}
  =
  \mathbb E_{\nu}\{w_\pi(X)r(X)\}.
\]

\subsection{Adjoint Bellman equation and balance moments}

The target occupancy measure \(\mu_\pi\) satisfies the adjoint Bellman equation
\citep{ueharaEtAl2021FiniteSampleMinimax}
\begin{equation}
\label{eqn::adjoint}
  \mu_\pi
  =
  (1-\gamma)\mu_{0,\pi}
  +
  \gamma\,\mu_\pi P_\pi .
\end{equation}
For any weight function \(\omega\) such that
\((\omega\nu)P_\pi\ll\nu\), define the adjoint Bellman operator
\[
  \mathcal B_\pi^\star \omega
  :=
  (1-\gamma)\omega_0
  +
  \gamma
  \frac{d\{(\omega\nu)P_\pi\}}{d\nu},
  \qquad
  \omega_0
  :=
  \frac{d\mu_{0,\pi}}{d\nu}.
\]
The occupancy ratio \(w_\pi=d\mu_\pi/d\nu\) is then characterized by the
fixed-point equation
\begin{equation}
\label{eq:adjoint-bellman-fixed-point}
  w_\pi
  =
  \mathcal B_\pi^\star w_\pi .
\end{equation}

Although \(\mathcal B_\pi^\star\omega\) is generally unavailable pointwise,
its action against test functions can be evaluated using one-step transitions
\citep{van2026fitted}. For any measurable \(f\) for which the expectations
exist,
\begin{equation}
\label{eq:bellman-moment-update}
  \mathbb E_{\nu}
  \left[
    \{\mathcal B_\pi^\star\omega\}(X)f(X)
  \right]
  =
  (1-\gamma)\mathbb E_{\mu_{0,\pi}}\{f(X_0)\}
  +
  \gamma\mathbb E_{\nu}
  \left[
    \omega(X)f(X^+)
  \right],
\end{equation}
where \(X\sim\nu\), \(X^+\mid X\sim P_\pi(\cdot\mid X)\), and
\(X_0\sim\mu_{0,\pi}\). Evaluating this identity at the fixed point yields the
occupancy-balance moments
\citep{nachumEtAl2019DualDICE,ueharaEtAl2021FiniteSampleMinimax}
\begin{equation}
\label{eqn::momentident}
  \mathbb E_{\nu}
  \left[
    w_\pi(X)\{g(X)-\gamma g(X^+)\}
  \right]
  =
  (1-\gamma)\mathbb E_{\mu_{0,\pi}}\{g(X_0)\},
\end{equation}
for every measurable \(g:\mathcal X\to\mathbb R\) for which the expectations
exist.

\subsection{Occupancy-ratio estimation and residual imbalance}

A common approach to occupancy-ratio estimation is minimax occupancy balancing
\citep{liuEtAl2018InfiniteHorizonOPE,nachumEtAl2019DualDICE,
ueharaEtAl2020MWLMQL,ueharaEtAl2021FiniteSampleMinimax}. These methods estimate
the occupancy ratio by minimizing violations of
\eqref{eqn::momentident} over a critic class. Schematically, they solve
\begin{equation}
\label{eq:minimax}
  \widehat\omega
  \in
  \argmin_{\omega\in\mathcal W}
  \sup_{f\in\mathcal F}
  \left|
    (1-\gamma)\mathbb E_{\mu_{0,\pi}}\{f(X_0)\}
    -
    \mathbb E_{\nu}
    \left[
      \omega(X)\{f(X)-\gamma f(X^+)\}
    \right]
  \right|,
\end{equation}
using empirical expectations and suitable regularization in practice. The
critic class \(\mathcal F\) determines which violations of the adjoint Bellman
equation are detectable. Thus, converting a small minimax objective into
control of the adjoint Bellman residual generally requires \(\mathcal F\) to
contain witnesses for the residuals generated by candidate weights. A
representative completeness condition is
\[
  \left\{
    \mathcal B_\pi^\star\omega-\omega:
    \omega\in\mathcal W
  \right\}
  \subseteq c\mathcal F
\]
for some \(c>0\)
\citep{ueharaEtAl2021FiniteSampleMinimax}. Without sufficient critic richness,
a candidate weight may satisfy every tested moment while remaining imbalanced
in untested directions. Even when the population classes are sufficiently
rich, finite-sample error, regularization, and imperfect optimization may
leave residual imbalance.

Fitted occupancy-ratio evaluation (\textsc{FORE})
\citep{van2026fitted} instead estimates the ratio through an iterative fitted
fixed-point procedure in which each update approximates the composition of an
adjoint Bellman update with a KL projection onto the ratio class. This avoids both
a separately specified critic class and closure of the ratio class under
adjoint Bellman updates. Nevertheless, statistical error, early stopping, and
the absence of a canonical held-out loss for tuning or model selection may
leave the fitted ratio imperfectly balanced. These considerations motivate a
post-processing step that directly targets residual Bellman imbalance in an
initial occupancy-ratio estimate.

\section{Bellman Calibration for Marginalized Importance Weights}

\subsection{Population Bellman calibration}

An estimated occupancy ratio may retain useful predictive information even
when its fitted values exhibit systematic Bellman imbalance. This is common
for flexible machine-learning estimators whose predictions preserve relative
structure but are distorted in scale or shape
\citep{niculescu2005predicting,guo2017calibration,dwivedi2020stable}.
We therefore post-process an initial estimate
\(\widehat w_\pi\) through a flexible transformation
\(\theta:\mathbb R\to\mathbb R_+\), yielding
\(\bar w_\pi=\theta\circ\widehat w_\pi\).
Such transformations include familiar corrections such as rescaling and
clipping. More generally, our goal is to remove Bellman imbalance detectable
from the fitted values while retaining the information they contain about the
true occupancy ratio. Rather than refitting the underlying high-dimensional
predictor or restricting attention to prespecified corrections, we formalize
this objective by imposing the adjoint Bellman moment identities against test
functions of the calibrated weight itself.

We say that \(\bar w_\pi\) is \textit{perfectly calibrated for occupancy
balance} if, for every bounded measurable \(g:\mathbb R\to\mathbb R\),
\begin{equation}
\label{eqn:perfect-occupancy-calibration}
  \mathbb E_{\nu}
  \left[
    \bar w_\pi(X)
    \{g(\bar w_\pi(X))-\gamma g(\bar w_\pi(X^+))\}
  \right]
  =
  (1-\gamma)\mathbb E_{\mu_{0,\pi}}
  \{g(\bar w_\pi(X_0))\}.
\end{equation}
Equivalently, the adjoint Bellman moment identities in
\eqref{eqn::momentident} hold with \(w_\pi\) replaced by \(\bar w_\pi\) for
all test functions of the form \(g\circ\bar w_\pi\).

To characterize this condition directly through the fitted values, define the
calibration function of a candidate weight \(\omega\) by
\begin{equation}
\label{eq:population-calibration-function}
  \Gamma_\omega(t)
  :=
  \mathbb E_{\nu}
  \left[
    \{\mathcal B_\pi^\star\omega\}(X)
    \,\middle|\,
    \omega(X)=t
  \right].
\end{equation}
Equation~\eqref{eqn:perfect-occupancy-calibration} is equivalent to
\[
\begin{aligned}
  &\mathbb E_{\nu}
    \left[
      \{\mathcal B_\pi^\star\bar w_\pi-\bar w_\pi\}(X)
      g(\bar w_\pi(X))
    \right]=0
    &&\text{for every bounded measurable }g,\\
  &\Longleftrightarrow\quad
    \bar w_\pi=\Gamma_{\bar w_\pi}\circ\bar w_\pi
    &&\nu\text{-almost surely}.
\end{aligned}
\]
Thus, a perfectly calibrated weight satisfies the adjoint Bellman fixed-point
equation after conditioning on its own fitted values. Equivalently, among
state--action pairs assigned the same fitted occupancy ratio, the average
adjoint Bellman update equals that fitted ratio.

The true occupancy ratio \(w_\pi\) is automatically perfectly calibrated:
since \(w_\pi=\mathcal B_\pi^\star w_\pi\) \(\nu\)-almost surely, the law
of iterated expectations yields
\(\Gamma_{w_\pi}\circ w_\pi=w_\pi\) \(\nu\)-almost surely. Perfect
calibration is weaker than pointwise recovery of \(w_\pi\): it requires only
that the Bellman residual vanish on average conditional on the fitted value.
It therefore provides a natural one-dimensional self-consistency condition
for post-processing fitted weights. In particular, a large fitted weight must
be matched by a correspondingly large conditional mean adjoint Bellman update.

\subsection{Calibration, optimal post-processing, and accuracy}
\label{sec:calrefine}

We next show that calibration ensures that an estimated occupancy ratio
effectively uses the information its fitted values contain about the true
ratio. A calibration--refinement bound makes this precise by separating
Bellman calibration error from the residual error achievable by optimal
post-processing. The key implication is that a well-calibrated estimate is
accurate whenever its fitted values can be accurately transformed into the
true occupancy ratio.

We define the \(L^2\) calibration and refinement errors by
\begin{equation}
\label{eq:l2-calibration-refinement-errors}
  \mathrm{Cal}_2^2(\omega)
  :=
  \|
    \Gamma_\omega\circ\omega-\omega
  \|_{L^2(\nu)}^2,
  \qquad
  \mathrm{Ref}_2^2(\omega)
  :=
  \inf_a
  \|
    w_\pi-a\circ\omega
  \|_{L^2(\nu)}^2,
\end{equation}
where the infimum is over measurable transformations
\(a:\mathbb R\to\mathbb R\). When \(w_\pi\in L^2(\nu)\), the optimal
post-processing is \(a\circ\omega=\mathbb E_{\nu}(w_\pi\mid\omega)\).
The calibration error also admits the dual Bellman-balance representation
\begin{equation}
\label{eq:occupancy-calibration-dual}
  \mathrm{Cal}_2(\omega)
  =
  \sup_{\mathbb E_{\nu}[g\{\omega(X)\}^2]\leq 1}
  \left|
    \mathbb E_{\nu}
    \left[
      \omega(X)
      \{g(\omega(X))-\gamma g(\omega(X^+))\}
    \right]
    -
    (1-\gamma)\mathbb E_{\mu_{0,\pi}}
    \{g(\omega(X_0))\}
  \right|.
\end{equation}
Thus, \(\mathrm{Cal}_2(\omega)\) is the largest normalized Bellman imbalance
detectable from the fitted values. It vanishes if and only if \(\omega\) is
perfectly calibrated for occupancy balance.

\begin{theorem}[Calibration--refinement bound for occupancy ratios]
\label{thm:population-l2-calibration-improvement}
For every \(\omega\in L^2(\nu)\) such that
\(\mathcal B_\pi^\star\omega\in L^2(\nu)\) and
\(w_\pi\in L^2(\nu)\),
\[
  \|\omega-w_\pi\|_{L^1(\nu)}
  \leq
  \frac{
    \mathrm{Cal}_2(\omega)+\mathrm{Ref}_2(\omega)
  }{1-\gamma}.
\]
\end{theorem}

The refinement error is the smallest \(L^2(\nu)\) discrepancy between the
true occupancy ratio and any transformation of the fitted weight
\(\omega(X)\). It therefore measures the best accuracy attainable by
post-processing that weight alone.
Theorem~\ref{thm:population-l2-calibration-improvement} bounds the error of
\(\omega\) by this optimal post-processing error plus its calibration error,
with the usual \((1-\gamma)^{-1}\) amplification of Bellman error in
discounted infinite-horizon problems
\citep{munosSzepesvari2008FVI}. Thus, a small calibration error guarantees
that \(\omega\) is accurate whenever some transformation of \(\omega(X)\) can
accurately recover \(w_\pi\). In particular, a perfectly calibrated weight
satisfies $\|\omega-w_\pi\|_{L^1(\nu)}
    \leq (1-\gamma)^{-1} \mathrm{Ref}_2(\omega).$ The \(L^1(\nu)\) error also has a direct functional interpretation:
\[
  \|\omega-w_\pi\|_{L^1(\nu)}
  =
  \sup_{\|g\|_\infty\leq 1}
  \left|
    \mathbb E_{\nu}\{\omega(X)g(X)\}
    -
    \mathbb E_{\mu_\pi}\{g(X)\}
  \right|.
\]
Consequently, the same bound controls the error in every bounded linear
functional of the target discounted occupancy distribution, including
expected rewards, costs, feature moments, and visitation probabilities.

This result is the occupancy-ratio analogue of classical
calibration--refinement decompositions for proper scoring rules
\citep{murphy1973new,degroot1983comparison,brocker2009reliability,
van2023causal}, and the adjoint counterpart of the Bellman
calibration--refinement decomposition for value functions in
\citet{van2025bellman}.

\section{Fitted isotonic Bellman calibration}
\label{sec:fittedalgo}

Motivated by the calibration--refinement bound, we develop a post-processing
procedure that improves Bellman calibration while preserving the ordering
induced by an initial occupancy-ratio estimate \(\widehat w_\pi\).  
The method applies \textsc{FORE} over the class of nondecreasing
transformations of \(\widehat w_\pi\), reducing each update to a
one-dimensional isotonic risk-minimization problem. Monotonicity preserves the
ranking induced by \(\widehat w_\pi\) while allowing its scale and shape to
adapt to the adjoint Bellman equation. The transformation class contains the
identity map and familiar restricted adjustments such as rescaling and
clipping, while also allowing flexible nonlinear corrections. As shown in
Section~\ref{sec:finite-sample-theory}, the resulting estimator achieves
Bellman calibration with a KL accuracy guarantee up to statistical error.

We separate estimation and calibration by sample splitting or cross-fitting,
following \citet{van2023causal}. The formal results below use an external calibration
sample independent of the training data. It contains one-step transitions
\(O_i=(S_i,A_i,S_i')\), \(i=1,\ldots,n\), with \((S_i,A_i)\sim\nu\) and
\(S_i'\mid S_i,A_i\sim P(\cdot\mid S_i,A_i)\). Given \(S_i'\), draw
\(A_i^+\sim\pi(\cdot\mid S_i')\), and define
\(X_i=(S_i,A_i)\) and \(X_i^+=(S_i',A_i^+)\).  We also observe an independent
sample \(X_{0,j}\sim\mu_{0,\pi}\), \(j=1,\ldots,m\), from the target initial
state--action distribution. In the empirical implementation, the fitting data
are instead reused through cross-calibration: estimate \(\widehat w_\pi\) by
cross-fitting
\citep{zheng2011cross,chernozhukov2018double}, as is standard in double
reinforcement learning
\citep{kallusUehara2020DoubleRL,kallusUehara2022BreakingHorizonDRL}, and
calibrate the pooled out-of-fold predictions
\citep{van2023causal,van2025stabilized}. The finite-sample theorem below covers
the external-sample implementation; Section~\ref{sec:experiments} evaluates
the pooled cross-fitted estimator on independent evaluation trajectories.

We now define the \textit{isotonic Bellman calibration} procedure. Let
\(\mathcal F_{\rm iso}\) denote the class of nondecreasing functions
\(h:\mathbb R\to\mathbb R\), and let \(\mathcal F_{{\rm iso},n}\) be the
empirical subclass of nondecreasing step functions determined by their values
at the fitted-value design points
\(\{\widehat w_\pi(X_i)\}_{i=1}^n\) and extended constantly between and beyond
these points. Algorithm~\ref{alg:isotonic-bellman-calibration} repeatedly fits
a transformation in \(\mathcal F_{{\rm iso},n}\) to the current
occupancy-ratio estimate and normalizes the resulting weights. As in classical
isotonic calibration
\citep{zadrozny2002transforming,niculescu2005predicting}, each update is a
convex one-dimensional problem that, after a one-time sorting step, can be
solved in linear time using a generalized pooled adjacent violators algorithm
\citep{barlow1972isotonic,best2000minimizing,de2010isotone}. The procedure
terminates after \(K_{\max}\) iterations or when the relative empirical
\(L^2\) change falls below \(\varepsilon\), where
\(\|f\|_{n,2}^2:=n^{-1}\sum_{i=1}^n f(X_i)^2\). In practice,
\(K_{\max}\) may be chosen large and
\(\varepsilon\) set to machine precision or a statistical tolerance;
the geometric convergence of the underlying \textsc{FORE} recursion implies
that \(O(\log n)\) iterations suffice to reduce the iteration error below any
polynomial statistical tolerance \citep{van2026fitted}. Computational details
are given in Appendix~\ref{app:pava}.

\begin{algorithm}[t]
\caption{Isotonic Bellman calibration for marginalized importance weighting}
\label{alg:isotonic-bellman-calibration}
\begin{algorithmic}[1]
\Require Initial estimate \(\widehat w_\pi\), samples
\(\{(X_i,X_i^+)\}_{i=1}^n\) and \(\{X_{0,j}\}_{j=1}^m\), tolerance
\(\varepsilon>0\), and \(K_{\max}\).
\State Initialize
\(\widehat\omega^{(0)}(x)\gets1\) and \(K\gets K_{\max}\).
\For{\(k=0,\ldots,K_{\max}-1\)}
    \State\label{algline:isotonic-update} Compute
    \[
    \begin{aligned}
    \widehat h_{k+1}
    &\in
    \argmin_{h\in\mathcal F_{{\rm iso},n}}
    \Biggl\{
        \log\left[
            \frac{1}{n}\sum_{i=1}^n
            \exp\{h(\widehat w_\pi(X_i))\}
        \right]
        -(1-\gamma)\frac{1}{m}\sum_{j=1}^m
        h\{\widehat w_\pi(X_{0,j})\} \\
    &\qquad
        -\gamma\frac{1}{n}\sum_{i=1}^n
        \widehat\omega^{(k)}(X_i)
        h\{\widehat w_\pi(X_i^+)\}
    \Biggr\}.
    \end{aligned}
    \]
    \State Set
    \[
        \widehat\omega^{(k+1)}(x)
        \gets
        \frac{
            \exp[\widehat h_{k+1}\{\widehat w_\pi(x)\}]
        }{
            n^{-1}\sum_{i=1}^n
            \exp[\widehat h_{k+1}\{\widehat w_\pi(X_i)\}]
        }.
    \]
    \State \textbf{If }
    \(\|\widehat\omega^{(k+1)}-\widehat\omega^{(k)}\|_{n,2}
    \leq
    \varepsilon\|\widehat\omega^{(k)}\|_{n,2}\),
    set \(K\gets k+1\) and \textbf{break}.
\EndFor
\State \Return
\(\widehat{\bar w}_\pi\gets\widehat\omega^{(K)}\).
\end{algorithmic}
\end{algorithm}

\section{Finite-sample guarantees for isotonic Bellman calibration}
\label{sec:finite-sample-theory}

\subsection{Bellman calibration error}
\label{sec:occupancy-calibration-error}

In this subsection, we derive a finite-sample bound for the calibration error
in the calibration--refinement decomposition of
Section~\ref{sec:calrefine}. The result shows that the exact empirical balance
enforced by the isotonic update translates into population calibration under
weak conditions. In Section~\ref{sec:isotonic-kl-regret}, we complement this
result with an oracle inequality showing that enforcing calibration preserves,
and may improve, occupancy-ratio accuracy up to statistical error.

Throughout
this section, following \cite{chen2019information}, we assume that the transition and initial-state samples are
i.i.d. and mutually independent. The analysis could be extended to dependent
trajectory data under suitable mixing conditions by replacing the i.i.d.
concentration arguments with their mixing-sequence analogues.

The argument begins with an exact empirical balance identity implied by the
isotonic subproblem. For any bounded \(g:\mathbb R\to\mathbb R\), consider the
exponentially tilted density ratio
\[
    \omega_\varepsilon(x)
    =
    \frac{
        \widehat\omega^{(k+1)}(x)
        \exp\!\left[
            \varepsilon
            g\{\widehat\omega^{(k+1)}(x)\}
        \right]
    }{
        P_n\!\left[
            \widehat\omega^{(k+1)}
            \exp\!\left\{
                \varepsilon
                g(\widehat\omega^{(k+1)})
            \right\}
        \right]
    }.
\]
For all sufficiently small positive and negative \(\varepsilon\), this
perturbation preserves normalization and isotonicity across the fitted-value
design points, as well as any fitted zeros. It therefore defines a two-sided
feasible path through the optimizer. Differentiating the empirical objective
along this path and applying first-order optimality at \(\varepsilon=0\)
yields
\begin{equation}
\label{eqn:empirical-calibration-kkt}
  \frac{1}{n}\sum_{i=1}^n
  \left[
    \widehat\omega^{(k+1)}(X_i)
    g\{\widehat\omega^{(k+1)}(X_i)\}
    -
    \gamma\widehat\omega^{(k)}(X_i)
    g\{\widehat\omega^{(k+1)}(X_i^+)\}
  \right]
  =
  (1-\gamma)\frac{1}{m}\sum_{j=1}^m
  g\{\widehat\omega^{(k+1)}(X_{0,j})\}.
\end{equation}
Thus, every isotonic update satisfies an exact empirical occupancy-balance
condition. When successive iterates are close,
\(\widehat\omega^{(k+1)}\approx\widehat\omega^{(k)}\), this identity becomes
the empirical analogue of perfect occupancy calibration in
\eqref{eqn:perfect-occupancy-calibration}.

We next combine \eqref{eqn:empirical-calibration-kkt} with uniform
concentration over functions of the fitted weights to obtain a finite-sample
bound for the population calibration error. The analysis relies on the
following regularity conditions. Throughout, let \(N=n\wedge m\) and
\(M_N:=1\vee A_{\rm env}\log(eN)\), where \(A_{\rm env}>0\) is fixed. For
each \(N\), define
\[
  \mathcal W_N
  :=
  \left\{
    f\circ\widehat w_\pi:
    f\ \text{is nondecreasing},\quad
    0\le f\circ\widehat w_\pi\le M_N,\quad
    \frac12
    \le
    \mathbb E_{\nu}(f\circ\widehat w_\pi)
    \le
    2
  \right\}.
\]

\begin{enumerate}[label=\textbf{A\arabic*}, ref={A\arabic*}, leftmargin=1.5em, series=cond]

\item \label{cond:iso-coverage}
\textit{Subexponential initial coverage and one-step smoothing.}
There are constants \(0<K_0,K_+<\infty\), independent of \(N\), such that
\[
  \left\|\frac{d\mu_{0,\pi}}{d\nu}\right\|_{\psi_1}\le K_0,
  \qquad
  \sup_{N\ge1}\sup_{\omega\in\mathcal W_N}
  \left\|
    \frac{d\{(\omega\nu)P_\pi\}}{d\nu}
  \right\|_{\psi_1}
  \le K_+,
\]
where
\(\|Z\|_{\psi_1}:=\inf\{s>0:E_{\nu}\exp(|Z|/s)\le2\}\).

\item \label{cond:iso-bounded-iterates}
\textit{Boundedness.}
It holds almost surely that
\(0\le \widehat\omega^{(k)}(x)\le M_N\) for \(0\le k\le K\)
for \(\nu\)-almost every \(x\).

\item \label{cond:bv-residual}
\textit{Finite variation.}
There is a finite constant \(V_{\rm cal}\), independent of \(N\), such that,
almost surely, the truncated calibration function
\[
  t
  \mapsto
  \mathbb E_{\nu}\!\left[
    \{\mathcal B_\pi^\star\widehat\omega^{(K)}\}(X)
    \mid \widehat\omega^{(K)}(X)=t
  \right]\wedge M_N
\]
has total variation at most \(V_{\rm cal}\vee M_N\).
\end{enumerate}
Condition~\ref{cond:iso-coverage} holds, for example, when the initial density
ratio and the transition density relative to \(\nu\) are uniformly bounded,
and more generally permits unbounded induced densities with uniformly
exponential tails; see
Lemma~\ref{lem:calibration-bounded-kernel-smoothing}.
Condition~\ref{cond:iso-bounded-iterates} requires the isotonic solutions to
grow at most logarithmically and holds automatically under uniform boundedness.
Finally, Condition~\ref{cond:bv-residual} is standard in isotonic calibration
\citep{van2023causal,van2025stabilized,van2024doubly,van2025bellman} and rules
out highly irregular truncated calibration functions. It holds, for example,
when the function is uniformly bounded and monotone, or when it has bounded
support and is piecewise Lipschitz with finitely many jumps.

\begin{theorem}[Finite-sample Bellman calibration]
\label{thm:isotonic-calibration-error}
Assume Conditions~\ref{cond:iso-coverage}--\ref{cond:bv-residual}. Let
\(K\ge1\), put \(K_{\rm sm}=(1-\gamma)K_0+\gamma K_+\), and suppose that
\(A_{\rm env}>2K_{\rm sm}/3\).
For each sufficiently large \(N=n\wedge m\), conditional on the training data
used to fit \(\widehat w_\pi\), the following bound holds with probability at
least \(1-\delta\) over the calibration sample:
\[
  \mathrm{Cal}_2(\widehat\omega^{(K)})
  \le
  \kappa_{{\rm cal},N}\sqrt{\log(eN)}
  \left\{
     N ^{-1/3}
    +
    \sqrt{\frac{\log(1/\delta)}{N}}
    +
    \|\widehat\omega^{(K)}
      -
      \widehat\omega^{(K-1)}\|_{n,2}
      \right\},
\]
where
\(\|f\|_{n,2}^2=n^{-1}\sum_{i=1}^n f(X_i)^2\).
The constant may be chosen so that $ \kappa_{{\rm cal},N}
  \le
  C
  \{1+K_0+K_++V_{\rm cal}+M_N\}^{8},$ where \(C<\infty\) is universal.
\end{theorem}

The first two terms in Theorem~\ref{thm:isotonic-calibration-error} capture
statistical calibration error, while the final term, $ \bigl\|
        \widehat\omega^{(K)}
        -
        \widehat\omega^{(K-1)}
    \bigr\|_{n,2},$ captures the remaining iteration error. This term is directly observable and
therefore provides a computable stopping criterion. It arises because the
transition term in \eqref{eqn:empirical-calibration-kkt} uses
\(\widehat\omega^{(K-1)}\), whereas the remaining terms use
\(\widehat\omega^{(K)}\).

Once the difference between successive iterates is no larger than the
statistical error, the fitted weight achieves \(L^2\) calibration error of
order \((n\wedge m)^{-1/3}\), up to logarithmic terms and
high-probability factors. Equivalently, its squared calibration error is of
order \((n\wedge m)^{-2/3}\), matching the classical squared-error rate for
isotonic regression and calibration
\citep{groeneboom2014nonparametric,chatterjee2015risk,van2023causal}.
Combined with the calibration--refinement bound in
Section~\ref{sec:calrefine}, this implies that the \(L^1(\nu)\)
occupancy-ratio error is controlled by the best \(L^2(\nu)\) approximation
of the true ratio by a transformation of the calibrated weights, plus a
calibration contribution of order \((n\wedge m)^{-1/3}\), all scaled by
\((1-\gamma)^{-1}\).

\subsection{KL regret and accuracy preservation}
\label{sec:isotonic-kl-regret}

The preceding results establish that isotonic Bellman calibration achieves
small Bellman calibration error. We now derive a generalized-KL oracle
inequality relative to the best monotone transformation of the initial
estimate. Generalized KL divergence controls \(L^1(\nu)\) error and, for
bounded rewards, policy-value error; the argument uses the KL contraction of
\textsc{FORE} \citep{van2026fitted}.

The KL analysis requires fitted ratios to remain strictly positive because the
logarithmic loss is singular at zero \citep{van2026fitted}. Unconstrained
isotonic risk minimization, however, may assign zero weight to the leftmost
fitted block, a boundary phenomenon at the lower end of the observed
fitted-value range
\citep{barlow1972isotonic,dai2020bias,lim2025estimator}. Under standard
isotonic regularity conditions, this block occupies a shrinking region of
order \(N^{-1/3}\)
\citep{isotonicKnotPoints,IsotonicNumberJumps}. Although zero weights are
admissible for reward reweighting, they lead to infinite KL divergence.

For the KL analysis, we therefore consider a strictly positive,
constrained variant of
Algorithm~\ref{alg:isotonic-bellman-calibration}. Each update is computed over
the same empirical isotonic step-function class, but with the fitted
transformation constrained to \([\varepsilon_{\rm b},T_N]\), or equivalently
its logarithm constrained to
\([\log\varepsilon_{\rm b},\log T_N]\), where
\(T_N=1\vee A_{\rm env}\log(eN)\) and
\(0<\varepsilon_{\rm b}\le1\) is a fixed user-chosen lower bound. Denote
the resulting iterates by \(\widehat\omega_{\rm KL}^{(k)}\). When the
constraints are inactive, the constrained and unconstrained updates coincide.
With a small lower bound and the growing upper bound \(T_N\), the constraints
otherwise affect only boundary behavior. We therefore use the constrained
variant for the KL analysis but recommend the unconstrained procedure in
practice.

For measurable \(f\geq 0\) and \(g>0\), define the generalized KL divergence
\begin{equation}
\label{eq:generalized-kl-calibration-refinement}
  D_{\nu}^{\rm KL}(f\|g)
  :=
  \mathbb E_{\nu}
  \left[
    f(X)\log\frac{f(X)}{g(X)}-f(X)+g(X)
  \right].
\end{equation}
This reduces to the usual KL divergence when both \(f\) and \(g\) integrate to
one under \(\nu\). Define the following class of normalized ratios generated
by bounded transformations in the empirical isotonic step class:
\[
  \mathcal W_{{\rm b},N}
  =
  \left\{
    \frac{f\{\widehat w_\pi(\cdot)\}}
         {\mathbb E_{\nu}f\{\widehat w_\pi(X)\}}:
    f\in\mathcal F_{{\rm iso},n},\
    \varepsilon_{\rm b}\le f\circ\widehat w_\pi\le T_N
  \right\}.
\]
Define the corresponding KL approximation error by
\[
  \varepsilon_{{\rm iso},N}
  :=
  \inf_{v\in\mathcal W_{{\rm b},N}}
  D_{\nu}^{\rm KL}(v\|w_\pi).
\]

\begin{enumerate}[label=\textbf{A\arabic*}, ref={A\arabic*},
                  leftmargin=1.5em, resume=cond]
\item \label{cond:kl-target}
\textit{Soft lower-tail margin.}
The target occupancy ratio is positive \(\nu\)-almost surely, and there exist
constants \(A<\infty\) and \(\alpha>0\) such that, for every \(t\in(0,1]\),
\[
\nu\{x:0<w_\pi(x)\le t\}\le A t^\alpha.
\]
\end{enumerate}
Condition~\ref{cond:kl-target} is the soft-margin condition used in analyses
of fitted \textsc{FORE} \citep{van2026fitted}. It permits \(w_\pi\) to
approach zero provided its lower tail has polynomially small \(\nu\)-mass,
and contributes only a logarithmic factor to the statistical error. The
adjoint Bellman equation gives
\(w_\pi\ge(1-\gamma)\omega_0\) \(\nu\)-almost everywhere, so if
\(\operatorname*{ess\,inf}_{x\sim\nu}\omega_0(x)\ge c_0>0\), then
\(w_\pi\ge m_\star:=(1-\gamma)c_0\) and
Condition~\ref{cond:kl-target} holds for any finite \(\alpha>0\) with
\(A=m_\star^{-\alpha}\). In particular, if \(\nu=\mu_{0,\pi}\), then
\(m_\star=1-\gamma\).

\begin{theorem}[Isotonic \textsc{FORE} calibration KL regret]
\label{thm:isotonic-kl-regret}
Assume Conditions~\ref{cond:iso-coverage} and~\ref{cond:kl-target}, with the
smoothing supremum in Condition~\ref{cond:iso-coverage} extended to
\(\mathcal W_N\) and to the normalized ratios induced by all
\(f\in\mathcal F_{\rm iso}\) satisfying
\(\varepsilon_{\rm b}\le f\circ\widehat w_\pi\le T_N\). Conditional on the
training data used to fit \(\widehat w_\pi\), with probability at least
\(1-\delta\) over the calibration sample,
\[
\begin{aligned}
  D_{\nu}^{\rm KL}(\widehat\omega_{\rm KL}^{(K)}\|w_\pi)
  \le{}&
  C_N\left(\frac{1+\gamma}{2}\right)^K
  D_{\nu}^{\rm KL}(\widehat\omega_{\rm KL}^{(0)}\|w_\pi) \\
  &+\frac{C_N}{1-\gamma}\varepsilon_{{\rm iso},N}
  +\frac{C_N}{(1-\gamma)^2}\log^2(eN)
  \left\{N^{-2/3}+\frac{\log(1/\delta)}{N}\right\},
\end{aligned}
\]
where \(N=n\wedge m\). For universal finite exponents \(p,q\) and a finite
constant \(C_0=C_0(A,\alpha)\), the envelope factor may be chosen so that
\[
  C_N
  \le
  C_0(1+K_0+K_+)^q
  \left\{1+(T_N/\varepsilon_{\rm b})^2\right\}^{p}.
\]
\end{theorem}

Theorem~\ref{thm:isotonic-kl-regret} gives an oracle bound comprising a
geometrically decaying iteration error, a bounded isotonic approximation error,
and a statistical error from the calibration sample. Choosing
\(K\geq \log(N)/\log\{2/(1+\gamma)\}\) makes the geometric factor at most
\(N^{-1}\), so logarithmically many iterations suffice to make the iteration
error negligible. Moreover, because \(T_N\) grows logarithmically and
\(\varepsilon_{\rm b}\) is fixed, \(C_N\) contributes only a fixed power of
\(\log(eN)\). Whenever the identity transformation is admissible, the
approximation error satisfies
\[
  \varepsilon_{{\rm iso},N}
  \le
  D_{\nu}^{\rm KL}
  \left(
    \frac{\widehat w_\pi}
         {\mathbb E_{\nu}\{\widehat w_\pi(X)\}}
    \middle\|
    w_\pi
  \right),
\]
and may be strictly smaller when a monotone transformation corrects systematic
miscalibration.

Whenever the empirically normalized unconstrained updates remain in
\([\varepsilon_{\rm b},T_N]\) at every
iteration, the constrained and unconstrained procedures coincide. Combined
with Theorem~\ref{thm:isotonic-calibration-error}, this yields a simultaneous
calibration-and-accuracy guarantee: if the bounded-class approximation and
iteration errors are no larger than the statistical error, the estimator
achieves small Bellman calibration error and KL risk within statistical error
of the best normalized ratio in the monotone-transformation class.

\section{Experimental evaluation}
\label{sec:experiments}

We test whether isotonic Bellman calibration improves out-of-sample Bellman
balance and off-policy value estimation. We evaluate on D4RL MuJoCo and
InfiniteCartPole under a common protocol.

\subsection{Experimental design and evaluation metrics}
\label{sec:experiments-design}

We calibrate neural \textsc{FORE}, \textsc{DualDICE}
\citep{nachumEtAl2019DualDICE}, SCOPE-RL minimax weight learning (MWL)
\citep{ueharaEtAl2020MWLMQL,kiyohara2023scope}, and \textsc{NeuralDICE}
\citep{yangEtAl2020RegularizedLagrangianDICE}. For the three comparison
methods, we use the authors' implementations identified in
Appendix~\ref{app:experiments-estimators}. The hyperparameters for each base
estimator are fixed across tasks and replications; reference policy values do
not enter fitting or hyperparameter selection. Ten-fold grouped cross-fitting
produces out-of-fold fitted scores.
We fit one isotonic map to the pooled out-of-fold scores, apply it to the ten
fold-specific predictors, and aggregate their predictions pointwise by the
median. The uncalibrated base estimator fitted once on the full training
sample serves as the control.

For D4RL, we use twelve matched dataset--policy pairs from HalfCheetah,
Hopper, and Walker2d \citep{fuEtAl2020D4RL}, two sample sizes, and ten
independent replications. For InfiniteCartPole, we use four behavior policies
and ten independent replications. With four base estimators, the design gives
960 paired D4RL comparisons and 160 paired CartPole comparisons.

The primary evaluation metrics are absolute policy-value error and projected
Bellman calibration error. The latter measures the component of the population
calibration residual detected by indicator functions of fitted-ratio quantile
bins. We estimate it on behavior-policy data independent of fitting and
calibration, using sample splitting to remove the upward noise bias from
squaring one empirical moment vector. Appendix~\ref{app:experiments-endpoints}
gives the estimator. All comparisons are paired within
experimental setting, sample size, replication, and base estimator, with
cluster bootstrap confidence intervals.

\subsection{Results}

\paragraph{Calibration and policy-value accuracy.}
Calibration reduces both projected calibration error and policy-value error in
D4RL and CartPole
(Table~\ref{tab:calibration-main-results}). Mean policy-value error falls from
\(2.99\) to \(0.710\) in D4RL and from \(0.176\) to \(0.110\) in CartPole.

\begin{table}[t]
\centering
\caption{Cross-calibrated estimator versus the uncalibrated estimator fitted
on the full training sample. Lower is better. Projected calibration error is
evaluated on an independent evaluation sample using the sample-split estimator in
Appendix~\ref{app:experiments-endpoints}. \(\Delta\) is calibrated minus
uncalibrated; brackets give paired 95\% cluster bootstrap confidence
intervals.}
\label{tab:calibration-main-results}
\small
\resizebox{\textwidth}{!}{%
\begin{tabular}{lrrr@{\qquad}rrr}
\toprule
& \multicolumn{3}{c}{Projected calibration error}
& \multicolumn{3}{c}{Absolute policy-value error} \\
\cmidrule(lr){2-4} \cmidrule(lr){5-7}
Benchmark & Uncalibrated & Calibrated & \(\Delta\) [95\% CI]
& Uncalibrated & Calibrated & \(\Delta\) [95\% CI] \\
\midrule
D4RL & 16.1 & 0.120 & \(-16.0\;[-44.3,-1.02]\)
& 2.99 & 0.710 & \(-2.28\;[-3.28,-1.39]\) \\
CartPole & 0.00207 & \(1.96\!\times\!10^{-4}\)
& \(-0.00187\;[-0.00248,-0.00133]\)
& 0.176 & 0.110 & \(-0.0658\;[-0.0733,-0.0580]\) \\
\bottomrule
\end{tabular}%
}
\end{table}

\begin{figure}[t]
\centering
\includegraphics[width=\textwidth]{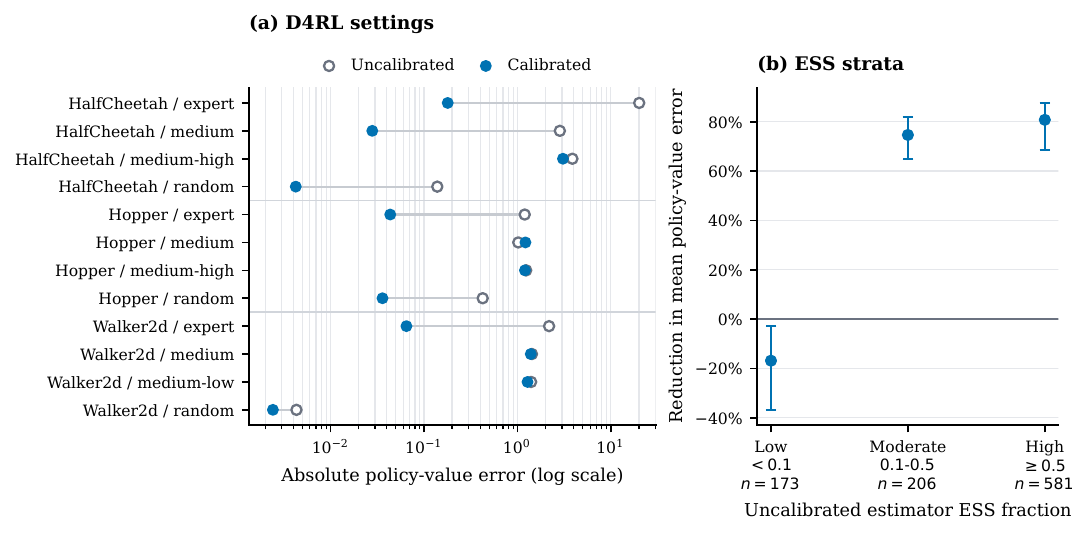}
\caption{D4RL policy-value error across tasks and ESS strata.
(a) Mean absolute policy-value error for each matched dataset--policy pair;
lines connect paired uncalibrated and calibrated estimators.
(b) Relative change in cluster-balanced mean policy-value error across strata
defined by the uncalibrated estimator's ESS fraction. Positive values favor
calibration; bars give paired 95\% cluster bootstrap confidence intervals.}
\label{fig:d4rl-settings-coverage}
\end{figure}

\paragraph{Task and ESS heterogeneity.}
Calibration lowers value error in eleven of twelve D4RL dataset--policy pairs
(Figure~\ref{fig:d4rl-settings-coverage}a). The direction and magnitude of the
change vary across ESS strata: mean policy-value error decreases by
80.8\% in the high-ESS stratum and 74.7\% in the moderate-ESS stratum, but
increases by 16.8\% in the low-ESS stratum
(Figure~\ref{fig:d4rl-settings-coverage}b).

\paragraph{Base estimators.}
Calibration reduces mean projected calibration error for every base estimator in
both benchmarks and reduces policy-value error for \textsc{DualDICE}, MWL, and
\textsc{NeuralDICE}. Neural \textsc{FORE} is the only exception: its projected
calibration error decreases, but its mean value error increases from \(0.465\) to \(1.03\)
in D4RL and from \(0.123\) to \(0.133\) in CartPole. The uncalibrated
\textsc{FORE} estimator has the lowest mean policy-value error among
the four base estimators in both benchmarks.
Estimator-level and task-level results appear in
Appendix~\ref{app:experiments-additional-results}.

\section{Conclusion}

An occupancy-ratio estimate may retain useful information about the true ratio
while remaining systematically inconsistent with adjoint Bellman balance.
Isotonic Bellman calibration corrects this residual imbalance through a
data-driven monotone transformation that adjusts the fitted weights while
preserving their ordering. This low-dimensional post-processing can improve
Bellman balance and, when the initial fitted values remain informative about
the true ratio, improve downstream off-policy value estimation.

The method is most useful when the behavior data adequately cover the target
occupancy and the initial ratio estimate preserves meaningful ranking
information. Under limited coverage or small calibration samples,
lower-complexity transformations, such as log-linear corrections, may provide
a better bias--variance tradeoff.

Several extensions are natural. First, Bellman calibration could be combined
with coverage-stopped FORE, calibrating occupancy-ratio estimates on the
well-supported region while avoiding extrapolation into poorly covered parts
of the target occupancy \citep{van2026fitted}. Second, calibration could be strengthened from the
one-dimensional tests induced by the fitted ratio to richer multicalibration
conditions involving additional state--action features or subgroups \citep{noarov2023scope}. Finally,
extending the framework to finite-horizon and nonstationary settings would
require time-indexed occupancy ratios and calibration conditions that may vary
across decision times.

\bibliographystyle{abbrvnat}
\bibliography{refs}

\appendix

\section{Finite-dimensional isotonic implementation}
\label{app:pava}

\subsection{Generalized PAVA}

This section derives the finite-dimensional form of the fitting step on
line~\ref{algline:isotonic-update} of
Algorithm~\ref{alg:isotonic-bellman-calibration} and describes its generalized
pooled adjacent violators algorithm (PAVA)
implementation. Fix an iteration \(k\). Denote the ordered distinct fitted
values observed in the behavior sample and their induced cells by
\[
  \mathcal T_n
  :=\left\{\widehat w_\pi(X_i):i=1,\ldots,n\right\}
  =\{t_1<\cdots<t_J\},
  \qquad
  I_j:=
  \begin{cases}
    (t_{j-1},t_j], & j=1,\ldots,J-1,\\
    (t_{J-1},\infty), & j=J,
  \end{cases}
\]
where \(t_0=-\infty\).
We represent the calibration transformation by a nondecreasing step function
that takes the value \(u_j=\exp\{h(t_j)\}\) on \(I_j\), allowing \(u_j=0\) as
a boundary value. Thus, the transformation is constant below the smallest
fitted-value design point, between successive design points, and above the
largest design point.

For \(j=1,\ldots,J\), define
\begin{equation}
\label{eq:pava-cell-masses}
\begin{aligned}
  a_j
  &=
  \frac{1}{n}\sum_{i=1}^n
  \mathbf 1\{\widehat w_\pi(X_i)=t_j\},\\
  b_j
  &=
  (1-\gamma)\frac{1}{m}\sum_{\ell=1}^m
  \mathbf 1\{\widehat w_\pi(X_{0,\ell})\in I_j\}
  +
  \gamma
  \frac{
    \sum_{i=1}^n
    \widehat\omega^{(k)}(X_i)
    \mathbf 1\{\widehat w_\pi(X_i^+)\in I_j\}
  }{
    \sum_{i=1}^n\widehat\omega^{(k)}(X_i)
  }.
\end{aligned}
\end{equation}
Thus, \(a_j\) is the empirical frequency of the fitted-value design point
\(t_j\) in the behavior sample, whereas \(b_j\) is the mass of the corresponding
fitted-value cell under the discounted mixture of the initial-state fitted
values and the self-normalized, \(\widehat\omega^{(k)}\)-weighted successor
fitted values. For the
empirically normalized iterates in
Algorithm~\ref{alg:isotonic-bellman-calibration},
\(\sum_i\widehat\omega^{(k)}(X_i)=n\); the denominator in
\eqref{eq:pava-cell-masses} is retained only to make the self-normalization
explicit. In particular, \(a_j>0\), \(b_j\ge 0\), and
\(\sum_j a_j=\sum_j b_j=1\).

With this notation, the fitting criterion reduces to
\begin{equation}
\label{eq:pava-log-objective}
  \min_{\substack{0\le u_1\le\cdots\le u_J\\
      \sum_{j=1}^J a_j u_j>0}}
  \left\{
    \log\left(\sum_{j=1}^J a_j u_j\right)
    -
    \sum_{j=1}^J b_j\log u_j
  \right\}.
\end{equation}
Here and below, \(0\log0=0\), whereas \(-b\log0=+\infty\) for
\(b>0\).
The objective is unchanged when all coordinates of \(u\) are multiplied by
the same positive constant. Its normalized representative is therefore
obtained by solving the separable convex problem
\begin{equation}
\label{eqn:pava-separable}
  \min_{0\le u_1\le\cdots\le u_J}
  \sum_{j=1}^J\{a_j u_j-b_j\log u_j\}.
\end{equation}
Indeed, for any fixed monotone vector \(u\), minimizing the criterion in
\eqref{eqn:pava-separable} over \(cu\), \(c>0\), gives
\(c=(\sum_j a_j u_j)^{-1}\). Substitution yields one plus the logarithmic
objective in \eqref{eq:pava-log-objective}. Hence, the solutions of
\eqref{eqn:pava-separable} are precisely the normalized solutions to the
fitting step in Algorithm~\ref{alg:isotonic-bellman-calibration}.

Problem \eqref{eqn:pava-separable} is a generalized isotonic regression
problem over a total order and can be solved by generalized PAVA
\citep{barlow1972isotonic,best2000minimizing,de2010isotone}. To obtain the
block update, consider a consecutive block
\(B\subseteq\{1,\ldots,J\}\) whose coordinates are constrained to share a
common value \(u\). Its contribution to the criterion is
\[
  A_Bu-C_B\log u,
  \qquad
  A_B=\sum_{j\in B}a_j,
  \qquad
  C_B=\sum_{j\in B}b_j.
\]
When \(A_B>0\), the blockwise minimizer is \(u_B=C_B/A_B\), with \(u_B=0\)
understood as a boundary value when \(C_B=0\). Because \(a_j>0\) for every
\(j\), each nonempty block has \(A_B>0\). Thus, every block update is well
defined, PAVA terminates after finitely many poolings, and the normalized
problem attains its minimum.

Generalized PAVA begins with singleton blocks and repeatedly pools adjacent
blocks whose block values violate monotonicity. Each pooled block is assigned
the updated value \(C_B/A_B\). Equivalently, this is ordinary weighted PAVA
applied to the pseudo-outcomes
\(b_j/a_j\) with weights \(a_j\).

Finally, the fitted calibration transformation takes the value \(\widehat u_j\)
on \(I_j\). When the fitted block values are positive, this is equivalent to
setting \(\widehat h=\log\widehat u_j\) on \(I_j\). Because the resulting
transformation is piecewise constant and nondecreasing as a function of
\(\widehat w_\pi\), the same finite-dimensional problem
can also be implemented using univariate regression-tree software that
supports monotonicity constraints and custom losses, including the
\texttt{R} and \texttt{Python} implementations of \texttt{xgboost}
\citep{chen2016xgboost}.

\subsection{Reference Python implementation}
\begin{verbatim}
import numpy as np

def pava(a, b, tol=1e-12):
    """Solve min_{u increasing} sum_j a_j u_j - b_j log u_j."""
    a = np.asarray(a, dtype=float).reshape(-1)
    b = np.asarray(b, dtype=float).reshape(-1)
    if a.shape != b.shape:
        raise ValueError("a and b must have the same shape")
    if np.any(a <= tol) or np.any(b < -tol):
        raise ValueError("a must be positive and b must be nonnegative")
    b = np.maximum(b, 0.0)

    def block_value(A, B):
        return B / A

    blocks = []

    for j in range(len(a)):
        A, B = float(a[j]), float(b[j])
        blocks.append([j, j + 1, A, B, block_value(A, B)])

        while len(blocks) > 1 and blocks[-2][4] > blocks[-1][4] + tol:
            r = blocks.pop()
            l = blocks.pop()
            A, B = l[2] + r[2], l[3] + r[3]
            blocks.append([l[0], r[1], A, B, block_value(A, B)])

    u = np.zeros_like(a)
    for start, stop, A, B, v in blocks:
        u[start:stop] = v

    return u / np.sum(a * u)


def isotonic_fore_weights(s, sp, s0, gamma, n_iter=50, tol=1e-10):
    """
    Exact isotonic FORE calibration on scalar fitted values.

    s  : fitted values at behavior-sample points X_i
    sp : fitted values at successor points X_i^+
    s0 : fitted values at initial-state points X_{0j}
    """
    n, m = len(s), len(s0)

    grid = np.unique(s)
    J = len(grid)

    def cell_index(x):
        return np.minimum(np.searchsorted(grid, x, side="left"), J - 1)

    ix = cell_index(s)
    ip = cell_index(sp)
    i0 = cell_index(s0)

    a = np.bincount(ix, minlength=J) / n
    b0 = np.bincount(i0, minlength=J) / m

    w = np.ones(n)

    for _ in range(n_iter):
        old = w

        bp = np.bincount(ip, weights=w / np.sum(w), minlength=J)
        b = (1.0 - gamma) * b0 + gamma * bp

        u = pava(a, b)
        w = u[ix]
        w = w / np.mean(w)

        if np.sqrt(np.mean((w - old) ** 2)) <= tol:
            break

    return w
\end{verbatim}

\section{Experimental details}
\label{app:experiments}

\subsection{Calibration and aggregation}
\label{app:experiments-cross-calibration}

\paragraph{\normalfont\bfseries Grouped cross-fitting.}
All base estimators use the same ten-fold trajectory-level partition, so
transitions and initial observations from a trajectory remain in the same
fold. For each fold \(k\), we fit the base ratio estimator on the other nine
folds and evaluate it on the held-out behavior observations \(X_i\),
target-policy successor observations \(X_i^+\), and initial observations
\(X_{0,j}\).

We pool the out-of-fold scores and fit a single isotonic calibration map. The
map is applied to each of the ten fold-specific predictors, and their
calibrated predictions are aggregated pointwise by the median. The comparison
estimator is the uncalibrated output of the same base estimator fitted on the
full training sample, using the same architecture and optimization schedule as
the fold-specific fits. The cross-calibrated estimator uses the normalized
isotonic map in Algorithm~\ref{alg:isotonic-bellman-calibration}.

\paragraph{\normalfont\bfseries Independent evaluation sample.}
Projected Bellman calibration error is evaluated on a behavior-policy sample
independent of all fitting and calibration data. For each fitted estimator, we
evaluate the resulting ratio on behavior observations \(X_i\), target-policy
successor observations \(X_i^+\), and initial observations \(X_{0,j}\). The
evaluation trajectories are partitioned into three disjoint subsamples.
Subsample \(C\) defines the estimator-specific quantile-bin basis and Gram
matrix, while subsamples \(A\) and \(B\) estimate the two Bellman-moment
vectors entering the Gram-weighted product estimator. All estimators in a
paired comparison use the same partition.

InfiniteCartPole uses independently simulated behavior trajectories for
evaluation. For D4RL, we split the logged episodes before constructing the
training and evaluation observations. We reserve \(20\%\) of episodes for
evaluation and use the remainder for training. Behavior, successor, and
initial-state observations are sampled from the corresponding episode sets,
so training and evaluation use disjoint episodes.

\paragraph{\normalfont\bfseries Isotonic map.}
The isotonic map has knots at the pooled observed behavior scores and is
constant outside their range. After PAVA, each boundary block containing fewer
than ten behavior observations is merged successively with the adjacent block
using behavior-frequency weights. This preserves monotonicity and total fitted
mass while leaving interior blocks unchanged. The resulting map is applied to
all fold-specific predictors.

\paragraph{\normalfont\bfseries Preprocessing of base-estimator outputs.}
The calibration method accepts any real-valued initial score. For base
estimators producing signed ratio estimates, we replace negative estimates by
zero in both the cross-fitted and full-sample estimators. Neural
\textsc{FORE} instead returns log-ratio scores. Within each fold, we truncate
these scores to the finite range observed for the held-out behavior
observations; the full-sample estimator uses the range observed in its
training data. The same bounds are used for successor, initial-state, and
independent evaluation observations.

\subsection{Base estimators and benchmarks}
\label{app:experiments-estimators}

\paragraph{\normalfont\bfseries Estimator settings.}
We use one fixed set of hyperparameters per base estimator across tasks,
replications, folds, and full-sample fits
(Table~\ref{tab:calibration-estimator-schedules}).

DualDICE uses the Google Research implementation\footnote{\url{https://github.com/google-research/google-research/tree/master/dual_dice}},
MWL uses SCOPE-RL version 0.2.1\footnote{\url{https://github.com/hakuhodo-technologies/scope-rl}},
and NeuralDICE uses the Google Research DICE-RL implementation\footnote{\url{https://github.com/google-research/dice_rl}}.

\begin{table}[t]
\centering
\caption{Base occupancy-ratio estimators and fixed hyperparameters.}
\label{tab:calibration-estimator-schedules}
\small
\begin{tabular}{@{}p{0.23\textwidth}p{0.70\textwidth}@{}}
\toprule
Estimator & Settings \\
\midrule
Neural \textsc{FORE}
& Two hidden layers of width 128; at most 300 outer iterations; 30 variational
steps per iteration; batch size 8,192; learning rate \(10^{-3}\); weight decay
\(0.1\); early stopping by Bellman loss on a 20\% validation subset of the
training fold; early stopping begins after 10 iterations and occurs after 20
iterations without improvement. \\
\textsc{DualDICE}
& Google Research implementation; 5,000 updates; batch size 256. \\
SCOPE-RL MWL
& State--action minimax objective; 10,000 updates; hidden width 128; batch size
128; training-fold standardization; median kernel bandwidth. \\
\textsc{NeuralDICE}
& Google Research DICE-RL configuration; primal, dual, and normalization
regularizers \((0,1,1)\); positive \(\zeta\) output; reward term included;
two hidden layers of width 64; 5,000 updates; batch size 256; learning rate
\(10^{-4}\). \\
\bottomrule
\end{tabular}
\end{table}

\paragraph{\normalfont\bfseries D4RL MuJoCo.}
We use twelve matched behavior-dataset and target-policy pairs from D4RL
\citep{fuEtAl2020D4RL}: random, medium, medium-high, and expert for
HalfCheetah and Hopper, and random, medium-low, medium, and expert for
Walker2d. We evaluate sample sizes \(10{,}000\) and \(50{,}000\), discount
\(\gamma=0.99\), and ten independent replications, giving 960 comparisons
across the four base estimators.

\paragraph{\normalfont\bfseries Reference policy values.}
D4RL target values are estimated from independent target-policy rollouts. We
increase the rollout sample until the two-sided 95\% confidence interval has
half-width at most
\(0.02\max\{|\widehat V|,\widehat{\operatorname{sd}},1\}\), subject to a
maximum of 32,768 trajectories. Every reported reference value satisfies this
criterion.

\paragraph{\normalfont\bfseries InfiniteCartPole.}
InfiniteCartPole uses full-support behavior policies with policy parameter
\(\alpha\in\{0,0.33,0.66,0.9\}\) and target policy parameter \(\alpha=1\).
Each setting uses 400 training trajectories of horizon 250, a sample of
\(50{,}000\) transitions, discount \(\gamma=0.995\), and ten independent
replications, giving 160
comparisons across the four base estimators. Projected Bellman calibration
error is evaluated on 400 additional behavior-policy trajectories independent
of training. Reference policy values are estimated by independent target-policy
rollouts using the same precision criterion as for D4RL.

\subsection{Evaluation metrics and uncertainty quantification}
\label{app:experiments-endpoints}

\paragraph{\normalfont\bfseries Policy-value error.}
For an estimated occupancy ratio \(\widehat\omega\), we estimate policy value by
\(\widehat V_{\widehat\omega}
:=n^{-1}\sum_{i=1}^n\widehat\omega(S_i,A_i)R_i\).
We report absolute error relative to rollout-based reference values using the
training behavior sample.

\paragraph{\normalfont\bfseries Projected Bellman calibration error.}
This metric measures the component of the population Bellman calibration
residual detected by indicator functions of fitted-ratio quantile bins. We
estimate the squared \(L^2(\nu)\) norm of this projection with three disjoint
evaluation subsamples. Subsample \(C\) defines the bin basis and empirical Gram
matrix. We set
\[
J_0
=
\min\left\{
20,\,
\max\left(5,\left\lfloor n_{AB}^{1/3}\right\rfloor\right)
\right\},
\]
where \(n_{AB}\) is the number of evaluation transition observations outside subsample
\(C\). Empirical quantiles define up to \(J_0\) nonempty bins with boundaries
at observed fitted values; ties are not split and may reduce the number of
bins to \(J\le J_0\).

Let
\(\phi_C(x)
:=(\mathbbm 1\{\widehat\omega(x)\in I_{C,j}\})_{j=1}^J\)
denote the bin indicators defined by subsample \(C\), and let
\(\mathbb P_D\) and \(\mathbb P_{0,D}\) denote empirical averages over the
transition and initial-state observations in subsample \(D\), respectively. For
\(D\in\{A,B\}\), define
\[
\begin{aligned}
b_D
&=
(1-\gamma)\mathbb P_{0,D}\{\phi_C(X_0)\}
+\gamma\mathbb P_D\{\widehat\omega(X)\phi_C(X^+)\}
-\mathbb P_D\{\widehat\omega(X)\phi_C(X)\}, \\
G_C
&=
\mathbb P_C\{\phi_C(X)\phi_C(X)^\top\}.
\end{aligned}
\]

At the population level, for a fixed basis \(\phi\), write
\(r_\omega:=\Gamma_\omega\circ\omega-\omega\),
\(b_\phi:=\E_\nu(r_\omega\phi)\), and
\(G_\phi:=\E_\nu(\phi\phi^\top)\). The corresponding projected Bellman
calibration error is
\[
\mathcal E_\phi(\omega)
=
b_\phi^\top G_\phi^\dagger b_\phi
=
\left\|
\Pi_{\operatorname{span}(\phi)}r_\omega
\right\|_{L^2(\nu)}^2
\le
\mathrm{Cal}_2^2(\omega),
\]
where \(G_\phi^\dagger\) is the Moore--Penrose inverse. We estimate this
quantity by
\[
\widehat{\mathcal E}_{\mathrm{proj}}
=
b_A^\top
\left\{
G_C
+
10^{-8}\frac{\operatorname{tr}(G_C)}{J}I_J
\right\}^{-1}
b_B.
\]

The diagonal term stabilizes the matrix inverse numerically. Conditional on
the fitted estimator, bin basis, and \(G_C\), the vectors \(b_A\) and \(b_B\)
independently estimate the same Bellman-moment vector. Their product therefore
omits the variance term that would enter if one empirical moment vector were
squared. The product estimator is signed, so
\(\widehat{\mathcal E}_{\mathrm{proj}}\) can be negative in finite samples;
we report it without truncation.

\paragraph{\normalfont\bfseries Effective sample size.}
For D4RL, the ESS fraction summarizes the concentration of the weights from the
uncalibrated full-sample estimator,
\[
\frac{\operatorname{ESS}(\widehat\omega)}{n}
=
\frac{\left\{\sum_{i=1}^n\widehat\omega_i\right\}^2}
{n\sum_{i=1}^n\widehat\omega_i^2}.
\]
We define low, moderate, and high ESS strata by fractions below \(0.1\), from
\(0.1\) to \(0.5\), and at least \(0.5\), respectively. The strata are computed
from the uncalibrated estimator and used to summarize heterogeneity in the
paired value-error comparisons.

\paragraph{\normalfont\bfseries Confidence intervals.}
All primary comparisons pair the full-sample and calibrated estimators within
experimental setting, sample size, discount, replication, and base estimator.
We aggregate paired differences within clusters defined by benchmark,
experimental setting, and replication. Percentile 95\% confidence intervals
use 10,000 cluster bootstrap resamples.

\subsection{Additional results}
\label{app:experiments-additional-results}

\paragraph{\normalfont\bfseries Base estimators.}
Table~\ref{tab:calibration-method-results} disaggregates projected calibration
and policy-value error by base estimator.

\begin{table}[H]
\centering
\caption{Mean error by base estimator. Each arrow compares the uncalibrated estimator
with the ten-fold cross-calibrated estimator. Lower is better.}
\label{tab:calibration-method-results}
\small
\begin{tabular}{llcc}
\toprule
Benchmark & Base estimator & Projected calibration error & Policy-value error \\
\midrule
D4RL & Neural \textsc{FORE} & $63.9\rightarrow0.238$ & $0.465\rightarrow1.03$ \\
D4RL & \textsc{DualDICE} & $0.308\rightarrow0.0686$ & $1.90\rightarrow0.495$ \\
D4RL & SCOPE-RL MWL & $0.228\rightarrow0.113$ & $0.685\rightarrow0.579$ \\
D4RL & \textsc{NeuralDICE} & $0.0633\rightarrow0.0608$ & $8.91\rightarrow0.741$ \\
\addlinespace
CartPole & Neural \textsc{FORE}
& $6.05\!\times\!10^{-4}\rightarrow3.28\!\times\!10^{-4}$
& $0.123\rightarrow0.133$ \\
CartPole & \textsc{DualDICE}
& $0.00673\rightarrow1.40\!\times\!10^{-4}$
& $0.212\rightarrow0.124$ \\
CartPole & SCOPE-RL MWL
& $4.27\!\times\!10^{-5}\rightarrow1.19\!\times\!10^{-5}$
& $0.183\rightarrow0.127$ \\
CartPole & \textsc{NeuralDICE}
& $9.03\!\times\!10^{-4}\rightarrow3.06\!\times\!10^{-4}$
& $0.186\rightarrow0.0568$ \\
\bottomrule
\end{tabular}
\end{table}

\paragraph{\normalfont\bfseries Task-level heterogeneity.}

Table~\ref{tab:calibration-d4rl-tasks} reports results for every matched D4RL
dataset--policy pair, averaging over sample sizes, replications, and
base estimators. Calibration lowers mean policy-value error in eleven of the
twelve pairs.

\begin{table}[H]
\centering
\caption{D4RL error by matched dataset--policy pair, averaged over sample
sizes, replications, and base estimators.}
\label{tab:calibration-d4rl-tasks}
\small
\begin{tabular}{lrr}
\toprule
Dataset--policy pair & Uncalibrated & Calibrated \\
\midrule
HalfCheetah--expert & 20.1 & 0.179 \\
HalfCheetah--medium & 2.84 & 0.0280 \\
HalfCheetah--medium-high & 3.87 & 3.07 \\
HalfCheetah--random & 0.139 & 0.00425 \\
Hopper--expert & 1.20 & 0.0435 \\
Hopper--medium & 1.02 & 1.22 \\
Hopper--medium-high & 1.24 & 1.20 \\
Hopper--random & 0.425 & 0.0360 \\
Walker2d--expert & 2.18 & 0.0649 \\
Walker2d--medium & 1.43 & 1.39 \\
Walker2d--medium-low & 1.40 & 1.28 \\
Walker2d--random & 0.00432 & 0.00242 \\
\bottomrule
\end{tabular}
\end{table}

\section{Fixed-image KL calibration--refinement identity}
\label{sec:calrefineKL}

In this section, we give the generalized-KL analogue of the fixed-image
identity in Lemma~\ref{lem:fixed-image-calibration-refinement}.

For measurable \(f\geq0\) and \(g>0\), recall the generalized KL divergence
\(D_{\nu}^{\rm KL}(f\|g)\) from
\eqref{eq:generalized-kl-calibration-refinement}.
We use the adjoint Bellman KL discrepancy
\(D_{\nu}^{\rm KL}(\mathcal B_\pi^\star\omega\|\omega)\) to measure
departure from the adjoint Bellman fixed-point equation. Indeed, for any
\(\omega>0\) satisfying \(\mathbb E_{\nu}\{\omega(X)\}=1\), the
\(L^1(\nu)\) contraction of \(\mathcal B_\pi^\star\), its fixed-point
property, and Pinsker's inequality give
\[
  \|\omega-w_\pi\|_{L^1(\nu)}
  \leq
  \frac{
    \sqrt{
      2D_{\nu}^{\rm KL}
      (\mathcal B_\pi^\star\omega\|\omega)
    }
  }{1-\gamma}.
\]
Thus, the adjoint Bellman KL discrepancy controls the \(L^1(\nu)\) error of
the estimated occupancy ratio.

We decompose this discrepancy into calibration and refinement components:
\begin{equation}
\label{eq:kl-calibration-refinement-errors}
  \mathrm{Cal}_{\rm KL}(\omega)
  :=
  D_{\nu}^{\rm KL}
  (\Gamma_\omega\circ\omega\|\omega),
  \qquad
  \mathrm{Ref}_{\rm KL}(\omega)
  :=
  \inf_a
  D_{\nu}^{\rm KL}
  (\mathcal B_\pi^\star\omega\|a\circ\omega),
\end{equation}
where the infimum is over measurable transformations
\(a:\mathbb R\to(0,\infty)\). The following result gives an exact decomposition
of the adjoint Bellman KL discrepancy.

\begin{theorem}[Fixed-image KL calibration--refinement identity]
\label{thm:population-kl-calibration-improvement}
For every \(\omega\) for which the quantities in
\eqref{eq:kl-calibration-refinement-errors} are finite,
\begin{equation}
\label{eq:kl-fixed-image-calibration-refinement}
  D_{\nu}^{\rm KL}
  (\mathcal B_\pi^\star\omega\|\omega)
  =
  \mathrm{Ref}_{\rm KL}(\omega)
  +
  \mathrm{Cal}_{\rm KL}(\omega).
\end{equation}
\end{theorem}

With the adjoint Bellman image held fixed, the refinement error is the smallest
KL discrepancy attainable by transforming only the fitted value \(\omega(X)\).
The calibration error is the remaining removable component; in particular,
\(\mathrm{Cal}_{\rm KL}(\omega)=0\) if and only if \(\omega\) is perfectly
calibrated. Moreover,
\[
  D_{\nu}^{\rm KL}
  (\mathcal B_\pi^\star\omega\|\omega)
  -
  D_{\nu}^{\rm KL}
  (\mathcal B_\pi^\star\omega\|
    \Gamma_\omega\circ\omega)
  =
  \mathrm{Cal}_{\rm KL}(\omega).
\]
Thus, with the adjoint Bellman image \(\mathcal B_\pi^\star\omega\) held fixed,
the calibration error is exactly the maximal reduction in KL discrepancy
attainable through post-processing based only on the fitted values. A perfectly
calibrated weight therefore cannot be further improved by such post-processing.
This is the occupancy-ratio analogue of classical calibration--refinement
decompositions for proper scoring rules
\citep{murphy1973new,degroot1983comparison,brocker2009reliability,
van2023causal}.

\section{Auxiliary results for the finite-sample analysis}
\label{app:finite-sample-proofs}

This section is organized by proof role. We first introduce the
empirical-process notation and restate the fitted \textsc{FORE} guarantee used
in the KL analysis. We then collect the population calibration identities,
establish the KL-regret bound for the bounded isotonic update, and derive the
finite-sample Bellman calibration bound from the exact empirical balance
identity.

Throughout the appendix, the training sample used to fit \(\widehat w_\pi\) is
conditioned on. All probabilities are conditional on that sample unless stated
otherwise.

\subsection{Notation and empirical-process conventions}
\label{app:proof-conventions}

Let
\[
  P_nf=\frac1n\sum_{i=1}^n f(X_i),\qquad
  Q_n\varphi=\frac1n\sum_{i=1}^n\varphi(X_i,X_i^+),\qquad
  P_{0,m}f=\frac1m\sum_{j=1}^m f(X_{0,j}).
\]
Let \(Q_{\nu,\pi}\) be the law of \((X,X^+)\) under
\(X\sim\nu\) and \(X^+\mid X\sim P_\pi(\cdot\mid X)\).

For this subsection, let \(Z_1,\ldots,Z_N\) be independent observations with
common law \(P\), let
\(P_N=N^{-1}\sum_{i=1}^N\delta_{Z_i}\), and let
\(\sigma_1,\ldots,\sigma_N\) be independent Rademacher variables. For a class
\(\mathcal G\) of square-integrable functions under \(P\), define
\[
  \mathcal R_N(\mathcal G,r;P)
  =
  \mathbb E_{Z,\sigma}
  \sup_{\substack{g\in\mathcal G:\\
      \|g\|_{L^2(P)}\le r}}
  \left|
    \frac{1}{N}\sum_{i=1}^N\sigma_i g(Z_i)
  \right|.
\]

\subsection{Fitted FORE guarantee}

We restate the fitted \textsc{FORE} finite-sample bound of
\citet{van2026fitted} in the notation of this paper.

Let \(\sH\) be a class of log-ratio functions on \(\sX\), and define
\[
  \Lambda_{\nu}(h)=\log \mathbb E_{\nu}\{\exp h(X)\},
  \qquad
  \omega_h(x)=\exp\{h(x)-\Lambda_{\nu}(h)\},
  \qquad
  \sW=\{\omega_h:h\in\sH\}.
\]
Let
\[
  \varepsilon_{\sH}
  :=
  \inf_{v\in\sW}
  D_{\nu}^{\rm KL}(v\|w_\pi).
\]
Set
\[
  \sH^\circ
  =
  \{h-\mathbb E_{\nu} h(X):h\in\sH\},
  \qquad
  \mathcal H_\Delta
  =
  \{h_1-h_2:h_1,h_2\in\sH^\circ\}.
\]
Let \(Q_{\nu,\Delta}\) denote the law of \((X,X)\) for \(X\sim\nu\), and let
\(Q_{\nu,\pi}\) denote the law of \((X,X^+)\) for \(X\sim\nu\) and
\(X^+\mid X\sim P_\pi(\cdot\mid X)\). Define
\[
  \mathcal G_\times
  =
  \left\{
    (x,x^+)\mapsto f(x)h_\Delta(x^+):
    f\in\sW,\ h_\Delta\in\mathcal H_\Delta
  \right\}
\]
and
\[
  \mathfrak C_N(r)
  =
  \max\left\{
    \mathcal R_N(\mathcal H_\Delta,r;\nu),
    \mathcal R_N(\mathcal H_\Delta,r;\mu_{0,\pi}),
    \mathcal R_N(\mathcal G_\times,r;Q_{\nu,\Delta}),
    \mathcal R_N(\mathcal G_\times,r;Q_{\nu,\pi})
  \right\}.
\]
Define the fitted critical radius by
\begin{equation}
\label{eq:calibration-critical-radius}
  \mathfrak r_{N,\rm fit}
  =
  N^{-1/2}
  \vee
  \inf\left\{r>0:\mathfrak C_N(r)\le r^2\right\}.
\end{equation}

\begin{enumerate}[label=\textbf{F\arabic*}, ref={F\arabic*}, leftmargin=1.5em]
\item \label{cond:fore-class}
\textit{Closed convex log-ratio class.}
\(\sH\) is convex, closed, and totally bounded as a subset of \(L^2(\nu)\).

\item \label{cond:fore-bounded}
\textit{Bounded log class, initial coverage, and one-step smoothing.}
There exist a measurable set \(\mathcal X_R\) with
\(\nu(\mathcal X_R)=1\) and finite constants \(R,K_0,K_+\) such that
\[
\begin{aligned}
  \sup_{h\in\sH}
  \left|h(x)-\mathbb E_{\nu}h(X)\right|\le R,
  \qquad x\in\mathcal X_R,\\
  \left\|\frac{d\mu_{0,\pi}}{d\nu}\right\|_{\psi_1}\le K_0,
  \qquad
  \sup_{\omega\in\sW}
  \left\|
    \frac{d\{(\omega\nu)P_\pi\}}{d\nu}
  \right\|_{\psi_1}
  \le K_+ .
\end{aligned}
\]

\item \label{cond:fore-target}
\textit{Target lower tail and KL approximation.}
The target occupancy ratio is positive \(\nu\)-almost surely,
\(\varepsilon_{\sH}<\infty\), and there are constants \(0<A<\infty\) and
\(\alpha>0\) such that, for every \(t\in(0,1]\),
\[
  \nu\{x:0<w_\pi(x)\le t\}
  \le
  A t^\alpha .
\]
\end{enumerate}

\begin{theorem}[Fitted FORE with empirical normalization]
\label{thm:fore-restated}
Let \(\gamma\in[0,1)\). Assume Conditions~\ref{cond:fore-class}--
\ref{cond:fore-target} and suppose that \(0\in\sH\). Let
\(\widehat h_0,\ldots,\widehat h_K\in\sH\) be the log-ratio representatives
generated by the exact fitted \textnormal{\textsc{FORE}} recursion over \(\sH\),
initialized at \(\widehat h_0=0\).
Write
\[
  \widehat\Lambda_n(h)
  =
  \log\left\{
    \frac{1}{n}\sum_{i=1}^n \exp h(X_i)
  \right\},
  \qquad
  \widehat\omega^{(k)}(x)
  =
  \exp\{\widehat h_k(x)-\widehat\Lambda_n(\widehat h_k)\},
  \qquad 0\le k\le K .
\]
In particular, \(\widehat\omega^{(0)}\equiv1\).
The iterates are constructed from independent transition
draws from \(\nu\) and independent initial draws from \(\mu_{0,\pi}\), and
each empirical subproblem is solved exactly.
Then, with probability at least \(1-\delta\),
\begin{equation}
\label{eq:fore-restated-bound}
\begin{aligned}
  D_{\nu}^{\rm KL}(\widehat\omega^{(K)}\|w_\pi)
  \le{}&
  C_{\rm env}
  \left(\frac{1+\gamma}{2}\right)^K
  D_{\nu}^{\rm KL}(\widehat\omega^{(0)}\|w_\pi)\\
  &+\frac{C_{\rm env}}{1-\gamma}\varepsilon_{\sH}\\
  &+\frac{C_{\rm env}}{(1-\gamma)^2}
  \log^2(eN)
  \left\{
    \mathfrak r_{N,\rm fit}^2
    +\frac{\log(1/\delta)}{N}
  \right\},
\end{aligned}
\end{equation}
where \(N=n\wedge m\). For universal finite exponents \(p,q\),
\begin{equation}
\label{eq:fore-restated-envelope}
  C_{\rm env}
  \le
  C_0(A,\alpha)(1+K_0+K_+)^q(1+e^{2R})^p.
\end{equation}
\end{theorem}

\begin{proof}
Apply \citet[Theorem~4.2]{van2026fitted} with behavior distribution \(\nu\),
target occupancy ratio \(w_\pi\), target initial state--action distribution
\(\mu_{0,\pi}\), and log-ratio class \(\sH\).
Conditions~\ref{cond:fore-class}--\ref{cond:fore-target}
correspond, respectively, to the closed convex class, bounded log-envelope and
smoothing, and target lower-tail conditions of that theorem.
The quantity \(\varepsilon_{\sH}\) is its KL approximation error. Exact
empirical minimization and normalization give the required update, while
\(\omega_{\widehat h_k}\in\sW\) verifies uniform smoothing at every
population-normalized iterate.

The cited theorem assumes equal numbers of transition and initial observations,
but its proof treats the two empirical processes separately. Applying the two
bounds at sample sizes \(n\) and \(m\) and then setting \(N=n\wedge m\) gives
\eqref{eq:fore-restated-bound}. Tracking the constants in that proof gives
the envelope bound in \eqref{eq:fore-restated-envelope}.

The argument also applies to an exact recursion over a possibly
sample-dependent closed convex subclass of \(\sH\) that contains zero. The
empirical-process event is uniform over \(\sH\), and, on this event, the
projection argument applies deterministically to the realized subclass. The
approximation error is then computed over that subclass, while the critical
radius and envelope constants are inherited from \(\sH\). This proves the
restated guarantee.
\end{proof}

\subsection{Empirical-process tools}

We use Bousquet's version of Talagrand's maximal inequality
\citep{bousquet2002bennett} and Dudley's localized entropy bound
\citep{bartlettEtAl2005LocalRademacher,wainwright2019HighDimensionalStatistics}.

\begin{lemma}[Bousquet's inequality]
\label{lem:tool-bousquet}
Let \(\mathcal G\) be a countable class of measurable functions satisfying
\(Pg=0\), \(\|g\|_\infty\le b\), and \(Pg^2\le v\) for all
\(g\in\mathcal G\). Then, for every \(u\ge0\), with probability at least
\(1-e^{-u}\),
\[
  \sup_{g\in\mathcal G}|(P_N-P)g|
  \le
  \mathbb E\sup_{g\in\mathcal G}|(P_N-P)g|
  +
  \sqrt{
    \frac{2u}{N}
    \left\{
      v+2b\mathbb E\sup_{g\in\mathcal G}|(P_N-P)g|
    \right\}
  }
  +
  \frac{bu}{3N}.
\]
\end{lemma}

\begin{lemma}[Localized entropy bound for Rademacher averages]
\label{lem:localized-entropy-bound}
Let \(\mathcal G\) have uniform envelope \(M\), and suppose that, uniformly over
probability measures \(Q\),
\[
  \log N\{\epsilon,\mathcal G,L^2(Q)\}\le H(\epsilon).
\]
Set \(J(r):=\int_0^r\sqrt{1+H(\epsilon)}\,d\epsilon\), with the integral
truncated at \(M\). Then, up to a universal constant,
\[
  \mathcal R_N(\mathcal G,r;P)
  \lesssim
  \frac{J(r)}{\sqrt N}
  +
  \frac{M J(r)^2}{r^2N}.
\]
\end{lemma}

\subsection{Population calibration and refinement identities}

\begin{proof}[Proof of Theorem~\ref{thm:population-l2-calibration-improvement}]
Write \(\Pi_\omega f=\mathbb E_{\nu}(f\mid\omega)\). Since
\(\mathcal B_\pi^\star w_\pi=w_\pi\), the triangle inequality gives
\begin{equation}
\label{eq:l2-calibration-refinement-chain}
\begin{aligned}
  \|\omega-w_\pi\|_{L^1(\nu)}
  &\leq
  \|\omega-\Pi_\omega w_\pi\|_{L^1(\nu)}
  +
  \|\Pi_\omega w_\pi-w_\pi\|_{L^1(\nu)} \\
  &\leq
  \|\omega-\Pi_\omega\mathcal B_\pi^\star\omega\|_{L^1(\nu)}
  +
  \|\Pi_\omega(
    \mathcal B_\pi^\star\omega
    -\mathcal B_\pi^\star w_\pi)
  \|_{L^1(\nu)} \\
  &\qquad+
  \|\Pi_\omega w_\pi-w_\pi\|_{L^1(\nu)} \\
  &\leq
  \mathrm{Cal}_2(\omega)
  +
  \gamma\|\omega-w_\pi\|_{L^1(\nu)}
  +
  \mathrm{Ref}_2(\omega).
\end{aligned}
\end{equation}
The final step in \eqref{eq:l2-calibration-refinement-chain} uses the
\(L^1\)-contraction property of conditional expectation and
\begin{equation}
\label{eq:adjoint-bellman-l1-contraction}
  \|\mathcal B_\pi^\star\omega
      -\mathcal B_\pi^\star w_\pi\|_{L^1(\nu)}
  \leq
  \gamma\|\omega-w_\pi\|_{L^1(\nu)}.
\end{equation}
The left-hand side of \eqref{eq:adjoint-bellman-l1-contraction} is the
\(\nu\)-density of
\(\gamma\{(\omega-w_\pi)\nu\}P_\pi\), while a Markov kernel cannot increase
the total variation norm of a signed measure. The \(L^2\)-projection property
of conditional expectation also gives
\begin{equation}
\label{eq:true-ratio-projection-bound}
  \|\Pi_\omega w_\pi-w_\pi\|_{L^1(\nu)}
  \leq
  \|\Pi_\omega w_\pi-w_\pi\|_{L^2(\nu)}
  =
  \mathrm{Ref}_2(\omega).
\end{equation}
Combining \eqref{eq:l2-calibration-refinement-chain}--
\eqref{eq:true-ratio-projection-bound} and rearranging gives the inequality in
Theorem~\ref{thm:population-l2-calibration-improvement}.
\end{proof}

We use the adjoint Bellman \(L^2\) error
\(\|\mathcal B_\pi^\star\omega-\omega\|_{L^2(\nu)}\) to measure departure
from the adjoint Bellman fixed-point equation. Indeed, the
\(L^1(\nu)\) contraction of \(\mathcal B_\pi^\star\), its fixed-point
property, and the inequality
\(\|f\|_{L^1(\nu)}\leq\|f\|_{L^2(\nu)}\) give
\[
  \|\omega-w_\pi\|_{L^1(\nu)}
  \leq
  \frac{
    \|\mathcal B_\pi^\star\omega-\omega\|_{L^2(\nu)}
  }{1-\gamma}.
\]
Thus, the adjoint Bellman \(L^2\) error controls the \(L^1(\nu)\) error of
the estimated occupancy ratio.

\begin{lemma}[Fixed-image calibration--refinement identity]
\label{lem:fixed-image-calibration-refinement}
For every \(\omega\in L^2(\nu)\) such that
\(\mathcal B_\pi^\star\omega\in L^2(\nu)\), define
\[
  \mathrm{Ref}_{\mathrm{img},2}^2(\omega)
  :=
  \inf_a
  \|\mathcal B_\pi^\star\omega-a\circ\omega\|_{L^2(\nu)}^2,
\]
where the infimum is over measurable transformations
\(a:\mathbb R\to\mathbb R\). Then
\begin{equation}
\label{eq:l2-fixed-image-calibration-refinement}
  \|\mathcal B_\pi^\star\omega-\omega\|_{L^2(\nu)}^2
  =
  \mathrm{Ref}_{\mathrm{img},2}^2(\omega)
  +
  \mathrm{Cal}_2^2(\omega).
\end{equation}
Equivalently,
\begin{equation}
\label{eq:l2-fixed-image-postprocessing}
  \|\mathcal B_\pi^\star\omega-\omega\|_{L^2(\nu)}^2
  -
  \|\mathcal B_\pi^\star\omega
    -\Gamma_\omega\circ\omega\|_{L^2(\nu)}^2
  =
  \mathrm{Cal}_2^2(\omega).
\end{equation}
\end{lemma}

\begin{proof}
Let \(Y_\omega=\mathcal B_\pi^\star\omega(X)\) and
\(S_\omega=\omega(X)\). For every square-integrable measurable transform
\(a(S_\omega)\), conditional expectation gives
\begin{equation}
\label{eq:l2-conditional-decomposition}
\begin{aligned}
  \mathbb E_{\nu}\{Y_\omega-a(S_\omega)\}^2
  &=
  \mathbb E_{\nu}\{Y_\omega-\Gamma_\omega(S_\omega)\}^2 \\
  &\quad+
  \mathbb E_{\nu}\{\Gamma_\omega(S_\omega)-a(S_\omega)\}^2,
\end{aligned}
\end{equation}
because the cross term has conditional expectation zero given \(S_\omega\).
The infimum over \(a\) is attained at
\(a(S_\omega)=\Gamma_\omega(S_\omega)\). Taking the infimum in
\eqref{eq:l2-conditional-decomposition} gives
\eqref{eq:l2-fixed-image-calibration-refinement}; setting
\(a(S_\omega)=S_\omega\) gives \eqref{eq:l2-fixed-image-postprocessing}.
\end{proof}

\begin{proof}[Proof of Theorem~\ref{thm:population-kl-calibration-improvement}]
Let \(Y_\omega=\mathcal B_\pi^\star\omega\) and \(S_\omega=\omega(X)\). For
any positive measurable transformation \(a(S_\omega)\) with finite
divergence, the tower property gives
\begin{equation}
\label{eq:kl-conditional-decomposition}
  D_{\nu}^{\rm KL}(Y_\omega\|a\circ\omega)
  =
  D_{\nu}^{\rm KL}(Y_\omega\|\Gamma_\omega\circ\omega)
  +
  D_{\nu}^{\rm KL}(\Gamma_\omega\circ\omega\|a\circ\omega).
\end{equation}
Here the cross term is zero because
\(\log\{\Gamma_\omega(S_\omega)/a(S_\omega)\}\) is a function of \(S_\omega\)
and
\(\mathbb E_{\nu}\{Y_\omega\mid S_\omega\}=\Gamma_\omega(S_\omega)\).
The second divergence is nonnegative, so the infimum over
such transformations is attained at
\(a\circ\omega=\Gamma_\omega\circ\omega\).
Thus
\(\mathrm{Ref}_{\rm KL}(\omega)
=D_{\nu}^{\rm KL}(Y_\omega\|\Gamma_\omega\circ\omega)\). Taking
\(a(S_\omega)=S_\omega\) in \eqref{eq:kl-conditional-decomposition} gives the
decomposition in \eqref{eq:kl-fixed-image-calibration-refinement}.
\end{proof}

\paragraph{Relation to squared calibration error.}
A second-order Taylor expansion of \(z\mapsto z\log z\) shows that, if
\(\omega\) and \(\Gamma_\omega\circ\omega\) are bounded below by \(m\), then
\[
  D_{\nu}^{\rm KL}
  (\Gamma_\omega\circ\omega\|\omega)
  \le
  \frac{1}{2m}
  \|\Gamma_\omega\circ\omega-\omega\|_{L^2(\nu)}^2
  =
  \frac{1}{2m}\mathrm{Cal}_2^2(\omega).
\]

For a positive fitted ratio \(\omega\), define the true-ratio KL refinement
error by
\[
  \mathrm{Ref}_{\pi,{\rm KL}}(\omega)
  :=
  \inf_a
  D_{\nu}^{\rm KL}(w_\pi\|a\circ\omega),
\]
where the infimum is over measurable transformations
\(a:\mathbb R\to(0,\infty)\).

\begin{lemma}[KL calibration--refinement bound for occupancy ratios]
\label{lem:calibration-refinement-fixed-point}
Suppose that \(w_\pi>0\) \(\nu\)-almost surely. For every normalized
\(\omega>0\) for which \(\mathrm{Cal}_{\rm KL}(\omega)\) and
\(\mathrm{Ref}_{\pi,{\rm KL}}(\omega)\) are finite,
\[
  \|\omega-w_\pi\|_{L^1(\nu)}
  \leq
  \frac{1}{1-\gamma}
  \left\{
    \sqrt{2\,\mathrm{Cal}_{\rm KL}(\omega)}
    +
    \sqrt{2\,\mathrm{Ref}_{\pi,{\rm KL}}(\omega)}
  \right\}.
\]
\end{lemma}

\begin{proof}
Write \(\Pi_\omega f=\mathbb E_{\nu}(f\mid\omega)\) and set
\(q_\omega=\Pi_\omega w_\pi\). For every positive measurable transform
\(a\circ\omega\) with finite divergence, conditional expectation gives
\[
  D_{\nu}^{\rm KL}(w_\pi\|a\circ\omega)
  =
  D_{\nu}^{\rm KL}(w_\pi\|q_\omega)
  +
  D_{\nu}^{\rm KL}(q_\omega\|a\circ\omega).
\]
The cross term is
\(
  \mathbb E_{\nu}[
    \{w_\pi-q_\omega\}\log\{q_\omega/(a\circ\omega)\}
  ]=0
\)
because the logarithm is a function of \(\omega\). Hence
\[
  \mathrm{Ref}_{\pi,{\rm KL}}(\omega)
  =
  D_{\nu}^{\rm KL}(w_\pi\|q_\omega),
\]
with the infimum attained at
\(a\circ\omega=q_\omega\).

Because \(\omega\) is normalized, so are
\(\Gamma_\omega\circ\omega\), \(q_\omega\), and \(w_\pi\). Using
\(w_\pi=\mathcal B_\pi^\star w_\pi\), conditional-expectation contraction,
the \(L^1(\nu)\)-contraction property of the adjoint Bellman update, and
Pinsker's inequality gives
\[
\begin{aligned}
  \|\omega-w_\pi\|_{L^1(\nu)}
  &\le
  \|\omega-\Pi_\omega\mathcal B_\pi^\star\omega\|_{L^1(\nu)}
  +
  \|\Pi_\omega(
    \mathcal B_\pi^\star\omega
    -\mathcal B_\pi^\star w_\pi)
  \|_{L^1(\nu)} \\
  &\qquad+
  \|q_\omega-w_\pi\|_{L^1(\nu)} \\
  &\le
  \sqrt{2\,\mathrm{Cal}_{\rm KL}(\omega)}
  +
  \gamma\|\omega-w_\pi\|_{L^1(\nu)}
  +
  \sqrt{2\,\mathrm{Ref}_{\pi,{\rm KL}}(\omega)}.
\end{aligned}
\]
Rearranging gives the bound in
Lemma~\ref{lem:calibration-refinement-fixed-point}.
\end{proof}

For any bounded measurable \(g:\mathcal X\to\mathbb R\), define the target
functional and, for \(\omega\in L^1(\nu)\), its weighted analogue by
\[
  \Psi_\pi(g)
  :=\mathbb E_{\mu_\pi}\{g(X)\}
  =\mathbb E_{\nu}\{w_\pi(X)g(X)\},
  \qquad
  \Psi_\omega(g):=\mathbb E_{\nu}\{\omega(X)g(X)\}.
\]

\begin{corollary}[Target-occupancy functional error]
\label{cor:calibrated-target-functional}
For any \(\omega\in L^2(\nu)\) with
\(\mathcal B_\pi^\star\omega\in L^2(\nu)\) and
\(w_\pi\in L^2(\nu)\),
\[
  \sup_{\|g\|_\infty\leq 1}
  \left|
    \Psi_\omega(g)-\Psi_\pi(g)
  \right|
  \leq
  \frac{
    \mathrm{Ref}_2(\omega)
    +
    \mathrm{Cal}_2(\omega)
  }{1-\gamma}.
\]
Moreover, for any \(\omega\in L^1(\nu)\) such that
\(w_\pi>0\) and \(\omega\geq0\) \(\nu\)-almost surely and
\(D_{\nu}^{\rm KL}(\omega\|w_\pi)<\infty\),
\[
  \sup_{\|g\|_\infty\leq 1}
  \left|
    \Psi_\omega(g)-\Psi_\pi(g)
  \right|
  \leq
  \left\{
    2\left[
      \mathbb E_{\nu}\{\omega(X)\}+1
    \right]
    D_{\nu}^{\rm KL}
    (\omega\|w_\pi)
  \right\}^{1/2}.
\]
\end{corollary}

\begin{proof}[Proof of Corollary~\ref{cor:calibrated-target-functional}]
For every \(g\) with \(\|g\|_\infty\leq1\),
\begin{equation}
\label{eq:functional-error-l1-bound}
  \left|\Psi_\omega(g)-\Psi_\pi(g)\right|
  =
  \left|
    \mathbb E_{\nu}\{(\omega-w_\pi)(X)g(X)\}
  \right|
  \leq
  \|\omega-w_\pi\|_{L^1(\nu)} .
\end{equation}
Combining \eqref{eq:functional-error-l1-bound} with
Theorem~\ref{thm:population-l2-calibration-improvement} proves the
squared-error calibration bound in the corollary.

For the generalized-Pinsker claim, the scalar inequality
\begin{equation}
\label{eq:scalar-generalized-pinsker}
  \frac{(a-b)^2}{a+b}
  \le
  2\left\{a\log\left(\frac{a}{b}\right)-a+b\right\},
  \qquad a,b\ge0,
\end{equation}
holds with the usual extended-value conventions. To verify it when \(b>0\),
set \(t=a/b\) and subtract the left-hand side of
\eqref{eq:scalar-generalized-pinsker} from its right-hand side. The resulting
function satisfies
\[
  F(t):=2(t\log t-t+1)-\frac{(t-1)^2}{t+1},
  \qquad
  F(1)=F'(1)=0,
  \qquad
  F''(t)=\frac2t-\frac8{(t+1)^3}\ge0,
\]
where the final inequality follows from \((t+1)^3\ge4t\). The cases \(b=0\)
follow from the extended-value convention.

For the generalized-KL bound in
Corollary~\ref{cor:calibrated-target-functional}, fix
\(\omega\in L^1(\nu)\) such that
\(w_\pi>0\) and \(\omega\geq0\) \(\nu\)-almost surely and
\(D_{\nu}^{\rm KL}(\omega\|w_\pi)<\infty\).
Inequality~\eqref{eq:scalar-generalized-pinsker} and Cauchy--Schwarz give
\begin{equation}
\label{eq:generalized-pinsker-l1-bound}
\begin{aligned}
  \|\omega-w_\pi\|_{L^1(\nu)}^2
  &\le
  \mathbb E_{\nu}\{\omega(X)+w_\pi(X)\}
  \mathbb E_{\nu}
  \left\{
    \frac{\{\omega(X)-w_\pi(X)\}^2}
         {\omega(X)+w_\pi(X)}
  \right\}\\
  &\le
  2\left[\mathbb E_{\nu}\{\omega(X)\}+1\right]
  D_{\nu}^{\rm KL}(\omega\|w_\pi),
\end{aligned}
\end{equation}
where \(\mathbb E_{\nu}w_\pi=1\). Combining
\eqref{eq:functional-error-l1-bound} and
\eqref{eq:generalized-pinsker-l1-bound} proves the generalized-KL bound.
\end{proof}

\begin{lemma}[KL contraction for the adjoint Bellman update]
\label{lem:calibration-adjoint-kl-contraction}
For any normalized \(\omega,\widetilde\omega\),
\begin{equation}
\label{eq:adjoint-bellman-kl-contraction}
  D_{\nu}^{\rm KL}(
    \mathcal B_\pi^\star\omega
    \|
    \mathcal B_\pi^\star\widetilde\omega
  )
  \le
  \gamma D_{\nu}^{\rm KL}(\omega\|\widetilde\omega).
\end{equation}
In particular,
\begin{equation}
\label{eq:adjoint-bellman-kl-fixed-point-contraction}
  D_{\nu}^{\rm KL}(\mathcal B_\pi^\star\omega\|w_\pi)
  \le
  \gamma D_{\nu}^{\rm KL}(\omega\|w_\pi).
\end{equation}
\end{lemma}

\begin{proof}
In measure form,
\((\mathcal B_\pi^\star\omega)\nu=(1-\gamma)\mu_{0,\pi}
+\gamma(\omega\nu)P_\pi\). Therefore, joint convexity of KL divergence and
data processing under \(P_\pi\) give
\[
\begin{aligned}
  D_{\nu}^{\rm KL}(
    \mathcal B_\pi^\star\omega
    \|
    \mathcal B_\pi^\star\widetilde\omega
  )
  &\le
  \gamma
  D_{\rm KL}\{(\omega\nu)P_\pi\|(\widetilde\omega\nu)P_\pi\}\\
  &\le
  \gamma D_{\nu}^{\rm KL}(\omega\|\widetilde\omega).
\end{aligned}
\]
The data-processing inequality establishes
\eqref{eq:adjoint-bellman-kl-contraction}.
Taking \(\widetilde\omega=w_\pi\) and using its fixed-point property proves
\eqref{eq:adjoint-bellman-kl-fixed-point-contraction}.
\end{proof}

\begin{lemma}[KL fixed-point recursion for projected Bellman updates]
\label{lem:kl-fixed-point-projection-recursion}
Under the conditions of
Lemma~\ref{lem:calibration-adjoint-kl-contraction}, let
\(\omega,\omega^+\) be normalized. If, for some \(\varepsilon\ge0\),
\begin{equation}
\label{eq:kl-projection-one-step}
  D_{\nu}^{\rm KL}(\omega^+\|w_\pi)
  \le
  D_{\nu}^{\rm KL}(\mathcal B_\pi^\star\omega\|w_\pi)
  +
  \varepsilon,
\end{equation}
then
\[
  D_{\nu}^{\rm KL}(\omega^+\|w_\pi)
  \le
  \gamma D_{\nu}^{\rm KL}(\omega\|w_\pi)+\varepsilon .
\]
Consequently, if iterates \(\omega^{(k+1)}\) satisfy
\eqref{eq:kl-projection-one-step} with
\(\omega=\omega^{(k)}\) and errors \(\varepsilon_k\), then
\begin{equation}
\label{eq:kl-projection-recursion}
  D_{\nu}^{\rm KL}(\omega^{(K)}\|w_\pi)
  \le
  \gamma^K D_{\nu}^{\rm KL}(\omega^{(0)}\|w_\pi)
  +
  \sum_{k=0}^{K-1}\gamma^{K-1-k}\varepsilon_k .
\end{equation}
If \(\varepsilon_k\le\varepsilon\) for all \(k\), the summation term in
\eqref{eq:kl-projection-recursion} is at most
\((1-\gamma)^{-1}\varepsilon\).
\end{lemma}

\begin{proof}
The one-step bound \eqref{eq:kl-projection-one-step} and
Lemma~\ref{lem:calibration-adjoint-kl-contraction} give
\[
  D_{\nu}^{\rm KL}(\omega^+\|w_\pi)
  \le
  D_{\nu}^{\rm KL}(\mathcal B_\pi^\star\omega\|w_\pi)+\varepsilon
  \le
  \gamma D_{\nu}^{\rm KL}(\omega\|w_\pi)+\varepsilon .
\]
Iterating this one-step inequality gives
\eqref{eq:kl-projection-recursion} and completes the proof.
\end{proof}

\subsection{Proof of the isotonic \textnormal{\textsc{FORE}} KL-regret theorem}
\label{app:kl-regret-proof}

\begin{lemma}[Entropy of the isotonic log-calibration class]
\label{lem:isotonic-entropy}
Let \(\mathcal M_R\) be the class of nondecreasing functions
\(\mathbb R\to[-R,R]\). There is a universal constant \(C<\infty\) such that,
for every probability distribution \(Q\) on \(\mathbb R\) and every
\(0<\epsilon\le R\),
\begin{equation}
\label{eq:isotonic-log-entropy}
  \log N\{\epsilon,\mathcal M_R,L^2(Q)\}
  \le
  \log N_{[]}\{\epsilon,\mathcal M_R,L^2(Q)\}
  \le
  \frac{CR}{\epsilon}.
\end{equation}
Consequently, for the fixed initial occupancy-ratio estimate
\(\widehat w_\pi\),
\[
  \sup_Q
  \log N\{
    \epsilon,
    \{h\circ\widehat w_\pi:h\in\mathcal M_R\},
    L^2(Q)
  \}
  \le
  \frac{CR}{\epsilon},
\]
where the supremum is over probability distributions on \(\sX\).
\end{lemma}

\begin{proof}
The bracketing bound in \eqref{eq:isotonic-log-entropy} is Theorem~2.7.5 of
\citet{vanDerVaartWellner1996WeakConvergence}; the covering-number inequality
follows because bracketing numbers dominate covering numbers. For the
composition class, apply that theorem to the pushforward distribution of
\(\widehat w_\pi(X)\) under \(Q\). Composition with \(\widehat w_\pi\)
preserves the \(L^2\) bracket radius, so every bracket for \(\mathcal M_R\)
under the pushforward distribution induces an \(L^2(Q)\) bracket for
\(\{h\circ\widehat w_\pi:h\in\mathcal M_R\}\). This proves the entropy bound
for the composition class.
\end{proof}

\begin{lemma}[Localized complexity of bounded isotonic calibration]
\label{lem:explicit-isotonic-localized-complexity}
Let \(0<\tau\le M\), and let
\[
  \overline{\mathcal W}_N
  =
  \{u=\theta\circ\widehat w_\pi:
    \theta\text{ nondecreasing},\
    \tau\le u\le M\}.
\]
Define
\[
  \mathcal H_N=\{\log u:u\in\overline{\mathcal W}_N\},
  \qquad
  \mathcal H_N^\circ
  =
  \{h-\mathbb E_{\nu}h:h\in\mathcal H_N\},
  \qquad
  \mathcal H_{N,\Delta}
  =
  \{h_1-h_2:h_1,h_2\in\mathcal H_N^\circ\},
\]
and
\[
  \mathcal W_N^{\rm norm}
  =
  \left\{
    \frac{u}{\mathbb E_{\nu}u}:u\in\overline{\mathcal W}_N
  \right\}.
\]
For \(Q_{\nu,\Delta}\) and \(Q_{\nu,\pi}\) as in
\eqref{eq:calibration-critical-radius}, define
\[
  \mathcal G_{N,\times}
  =
  \{(x,x^+)\mapsto \omega(x)b(x^+):
    \omega\in\mathcal W_N^{\rm norm},\ b\in\mathcal H_{N,\Delta}\}.
\]
There is a constant \(C_{\tau,M}<\infty\) such that, for every \(r>0\),
\begin{equation}
\label{eq:isotonic-local-rademacher-bounds}
\begin{aligned}
  \mathcal R_N(\mathcal H_{N,\Delta},r;P)
  &\le
  C_{\tau,M}
  \left\{\sqrt{\frac{r}{N}}+\frac{1}{Nr}\right\},
  \qquad P\in\{\nu,\mu_{0,\pi}\},\\
  \mathcal R_N(\mathcal G_{N,\times},r;Q)
  &\le
  C_{\tau,M}
  \left\{\sqrt{\frac{r}{N}}+\frac{1}{Nr}\right\},
  \qquad Q\in\{Q_{\nu,\Delta},Q_{\nu,\pi}\}.
\end{aligned}
\end{equation}
Consequently, the corresponding critical radius satisfies
\[
  \mathfrak r_{N,\rm fit}
  \le
  C_{\tau,M}N^{-1/3}.
\]
Moreover, one may take
\[
  C_{\tau,M}
  \le
  C\{1+M/\tau\}^{4}
\]
for a universal finite \(C\).
\end{lemma}

\begin{proof}
Multiplying every raw weight by \(\tau^{-1}\) changes neither its centered
log-ratio representative nor its population-normalized ratio. We may therefore
work with raw weights in \([1,L]\), where \(L=M/\tau\). For any probability
measure \(Q\), Lemma~\ref{lem:isotonic-entropy}, applied to the pushforward law
of \(\widehat w_\pi(X)\), gives
\[
  \log N\{\epsilon,\mathcal H_N,L^2(Q)\}
  \le
  \frac{C(1+\log L)}{\epsilon}.
\]
To account for centering, set \(\overline Q=(Q+\nu)/2\). For
\(h,h'\in\mathcal H_N\),
\[
\begin{aligned}
  \|&(h-\mathbb E_{\nu}h)-(h'-\mathbb E_{\nu}h')\|_{L^2(Q)}
  \\
  \le
  \|h-h'\|_{L^2(Q)}+\|h-h'\|_{L^2(\nu)}
  \le
  2\sqrt 2\|h-h'\|_{L^2(\overline Q)}.
\end{aligned}
\]
Thus, the same uniform \(1/\epsilon\) entropy bound holds for
\(\mathcal H_N^\circ\) and, by a product-cover argument, for
\(\mathcal H_{N,\Delta}\).

The monotone bracketing bound on the weight scale gives
\[
  \log N\{\epsilon,\overline{\mathcal W}_N,L^2(Q)\}
  \le
  \frac{CL}{\epsilon}.
\]
Moreover, if \(\bar u=u/\mathbb E_{\nu}u\) and
\(\bar u'=u'/\mathbb E_{\nu}u'\), then
\[
  \|\bar u-\bar u'\|_{L^2(Q)}
  \le
  \|u-u'\|_{L^2(Q)}
  +
  L\|u-u'\|_{L^2(\nu)}.
\]
Applying the mixture argument with \((Q+\nu)/2\) shows that
\(\mathcal W_N^{\rm norm}\) also has uniform \(1/\epsilon\) entropy. Since
\(\mathcal W_N^{\rm norm}\) has envelope \(L\) and
\(\mathcal H_{N,\Delta}\) has envelope \(2\log L\), a product cover yields
\[
  \log N\{\epsilon,\mathcal G_{N,\times},L^2(Q)\}
  \le
  \frac{C(1+L)^3}{\epsilon}
\]
uniformly over \(Q\). Lemma~\ref{lem:localized-entropy-bound}, with
\(J(r)\le C(1+L)^{3/2}\sqrt r\), gives
\eqref{eq:isotonic-local-rademacher-bounds} with
\(C_{\tau,M}\le C(1+L)^4\).

Let
\[
  \mathfrak C_{N,\rm iso}(r)
  =
  \max\left\{
    \mathcal R_N(\mathcal H_{N,\Delta},r;\nu),
    \mathcal R_N(\mathcal H_{N,\Delta},r;\mu_{0,\pi}),
    \mathcal R_N(\mathcal G_{N,\times},r;Q_{\nu,\Delta}),
    \mathcal R_N(\mathcal G_{N,\times},r;Q_{\nu,\pi})
  \right\}.
\]
The Rademacher bounds imply \(\mathfrak C_{N,\rm iso}(r)\le r^2\) whenever
\(r\ge C_{\tau,M}N^{-1/3}\). Since
\(N^{-1/2}\le N^{-1/3}\),
the definition in \eqref{eq:calibration-critical-radius} gives
\(\mathfrak r_{N,\rm fit}\le C_{\tau,M}N^{-1/3}\), as claimed.
\end{proof}
 
\begin{proof}[Proof of Theorem~\ref{thm:isotonic-kl-regret}]
Set \(N=n\wedge m\) and
\(R_N:=\log(T_N/\varepsilon_{\rm b})\). For a measurable function \(h\),
define
\[
  \Lambda(h)
  :=
  \log\mathbb E_{\nu}\exp\{h(X)\},
  \qquad
  \omega_h
  :=
  \exp\{h-\Lambda(h)\}.
\]
Consider the bounded isotonic class
\[
  \overline{\mathcal H}_{{\rm b},N}
  :=
  \left\{
    g\circ\widehat w_\pi:
    g\ \text{is nondecreasing},\quad
    \log\varepsilon_{\rm b}
    \leq g\circ\widehat w_\pi
    \leq \log T_N
  \right\},
\]
and its empirical step-function subclass
\[
  \mathcal H_{{\rm b},N}
  :=
  \left\{
    g\circ\widehat w_\pi:
    g\in\mathcal F_{{\rm iso},n},\quad
    \log\varepsilon_{\rm b}
    \leq g\circ\widehat w_\pi
    \leq \log T_N
  \right\}.
\]
Because exponentiation preserves monotonicity,
\(\{\omega_h:h\in\mathcal H_{{\rm b},N}\}=\mathcal W_{{\rm b},N}\).

We first verify initialization and finiteness of the approximation error.
Since
\(\varepsilon_{\rm b}\leq1\leq T_N\), the zero function belongs to
\(\mathcal H_{{\rm b},N}\). Thus,
\(\widehat h_0=0\) and
\(\widehat\omega_{\rm KL}^{(0)}\equiv1\), as required by
Theorem~\ref{thm:fore-restated}. Moreover,
Condition~\ref{cond:kl-target} and Tonelli's theorem give
\[
  \mathbb E_{\nu}
  \left[
    \log\{1/w_\pi(X)\}
    \mathbf 1\{w_\pi(X)\leq1\}
  \right]
  \leq
  \int_0^\infty
  \{1\wedge Ae^{-\alpha s}\}\,ds
  <\infty.
\]
The opposite tail is also integrable because
\(\log w_\pi\leq w_\pi\) on \(\{w_\pi>1\}\) and
\(\mathbb E_{\nu}w_\pi=1\). Hence
\(D_{\nu}^{\rm KL}(1\|w_\pi)<\infty\).
Since \(1\in\mathcal W_{{\rm b},N}\), the infimum defining
\(\varepsilon_{{\rm iso},N}\) is finite.

We next verify the geometric conditions on the function classes. The class
\(\overline{\mathcal H}_{{\rm b},N}\) is convex because monotonicity and the
pointwise box constraints are preserved under convex combinations. To prove
closedness, let \(T=\widehat w_\pi(X)\), let \(P_T\) be its distribution under
\(\nu\), and suppose
\[
  h_j=g_j(T)\in\overline{\mathcal H}_{{\rm b},N},
  \qquad
  h_j\longrightarrow h
  \quad\text{in }L^2(\nu).
\]
The \(\sigma(T)\)-measurable functions form a closed subspace of
\(L^2(\nu)\), so \(h=g(T)\) almost surely for some measurable \(g\), and
\(g_j\to g\) in \(L^2(P_T)\). Along a subsequence, this convergence holds
pointwise on a set \(A\subseteq\mathbb R\) with \(P_T(A)=1\). Because every
\(g_j\) is nondecreasing, \(g\) is nondecreasing on \(A\). Define
\[
  \widetilde g(t)
  :=
  \sup\{g(s):s\in A,\ s\leq t\},
\]
with the supremum set equal to
\(\log\varepsilon_{\rm b}\) when the indexing set is empty. Then
\(\widetilde g\) is nondecreasing, agrees with \(g\) on \(A\), and takes
values in
\([\log\varepsilon_{\rm b},\log T_N]\). Therefore
\(h=\widetilde g(T)\) almost surely, proving that
\(\overline{\mathcal H}_{{\rm b},N}\) is closed in \(L^2(\nu)\).

Lemma~\ref{lem:isotonic-entropy} shows that
\(\overline{\mathcal H}_{{\rm b},N}\) is totally bounded. It is therefore
compact in \(L^2(\nu)\). For every
\(h\in\overline{\mathcal H}_{{\rm b},N}\),
\[
\begin{gathered}
  \log\varepsilon_{\rm b}\leq h(x)\leq\log T_N,
  \qquad
  \log\varepsilon_{\rm b}\leq\Lambda(h)\leq\log T_N,\\
  \|\log\omega_h\|_\infty
  =\|h-\Lambda(h)\|_\infty\leq R_N.
\end{gathered}
\]

If the behavior-sample fitted values have \(J\) distinct values, then
\(\mathcal H_{{\rm b},N}\) is parameterized by
\[
  \left\{
    (a_1,\ldots,a_J):
    \log\varepsilon_{\rm b}
    \leq a_1\leq\cdots\leq a_J
    \leq\log T_N
  \right\}.
\]
This parameter set is compact and convex, and the associated step-function
extension is continuous in \(L^2(\nu)\). Consequently,
\(\mathcal H_{{\rm b},N}\) is a closed, convex, compact subset of
\(\overline{\mathcal H}_{{\rm b},N}\) containing zero.

We now relate the algorithmic iterates to the population-normalized ratio
class. Define
\(\widehat\Lambda_n(h):=\log[n^{-1}\sum_{i=1}^n\exp\{h(X_i)\}]\). If
\(\widehat h_k\in\mathcal H_{{\rm b},N}\) is the fitted log-calibration
function at iteration \(k\), then
\(\widehat\omega_{\rm KL}^{(k)}
=\exp\{\widehat h_k-\widehat\Lambda_n(\widehat h_k)\}\).
Its population-normalized version is
\[
  \frac{\widehat\omega_{\rm KL}^{(k)}}
       {\mathbb E_{\nu}\widehat\omega_{\rm KL}^{(k)}}
  =
  \frac{\exp\{\widehat h_k\}}
       {\mathbb E_{\nu}\exp\{\widehat h_k(X)\}}
  =
  \omega_{\widehat h_k}
  \in
  \mathcal W_{{\rm b},N}.
\]
Thus, the restricted recursion in
Theorem~\ref{thm:fore-restated} has approximation error exactly
\(\varepsilon_{{\rm iso},N}\).

It remains to bound the statistical complexity. Apply
Lemma~\ref{lem:explicit-isotonic-localized-complexity} with
\[
  \tau=\varepsilon_{\rm b},
  \qquad
  M=T_N,
  \qquad
  \overline{\mathcal W}_N
  =
  \{\exp(h):h\in\overline{\mathcal H}_{{\rm b},N}\}.
\]
The lemma controls the centered difference class
\(\mathcal H_\Delta\) and the population-normalized product class
\(\mathcal G_\times\) in
\eqref{eq:calibration-critical-radius}, and gives
\[
  \mathfrak r_{N,\rm fit}^2
  \leq
  C_{\rm iso,N}N^{-2/3},
  \qquad
  C_{\rm iso,N}
  \leq
  C\{1+T_N/\varepsilon_{\rm b}\}^{8}.
\]

All conditions of Theorem~\ref{thm:fore-restated} and its uniform-subclass
extension are now verified. Condition~\ref{cond:iso-coverage} gives the
required smoothing bound uniformly over \(\mathcal W_N\) and the
population-normalized ratio class generated by
\(\overline{\mathcal H}_{{\rm b},N}\), while
Condition~\ref{cond:kl-target} gives the required lower-tail bound. Therefore,
with probability at least \(1-\delta\),
\[
\begin{aligned}
  D_{\nu}^{\rm KL}
  \bigl(\widehat\omega_{\rm KL}^{(K)}\|w_\pi\bigr)
  &\leq
  C_N\left(\frac{1+\gamma}{2}\right)^K
  D_{\nu}^{\rm KL}
  \bigl(\widehat\omega_{\rm KL}^{(0)}\|w_\pi\bigr)
  +
  \frac{C_N}{1-\gamma}\,
  \varepsilon_{{\rm iso},N}
  \\
  &\quad+
  \frac{C_N}{(1-\gamma)^2}
  \log^2(eN)
  \left\{
    N^{-2/3}
    +
    \frac{\log(1/\delta)}{N}
  \right\}.
\end{aligned}
\]
Finally, the polynomial bounds in
Theorem~\ref{thm:fore-restated} and
Lemma~\ref{lem:explicit-isotonic-localized-complexity} imply that, for
universal finite exponents \(p,q\),
\[
  C_N
  \leq
  C_0(A,\alpha)(1+K_0+K_+)^q
  \{1+(T_N/\varepsilon_{\rm b})^2\}^{p}.
\]
This proves the theorem.
\end{proof}

\subsection{Proof of finite-sample Bellman calibration}
\label{app:calibration-bound-proof}

\begin{proof}[Proof of the dual representation in
\eqref{eq:occupancy-calibration-dual}]
By \eqref{eq:bellman-moment-update},
\[
  \mathbb E_{\nu}[\{\mathcal B_\pi^\star\omega\}(X)g\{\omega(X)\}]
  =
  (1-\gamma)\mathbb E_{\mu_{0,\pi}}[g\{\omega(X_0)\}]
  +
  \gamma\mathbb E_{\nu}[\omega(X)g\{\omega(X^+)\}].
\]
Hence, the Bellman-balance functional in
\eqref{eq:occupancy-calibration-dual} equals
\[
  \mathbb E_{\nu}[
    \{\omega(X)-\mathcal B_\pi^\star\omega(X)\}g\{\omega(X)\}
  ].
\]
Let
\[
  r_\omega(X)
  =
  \mathbb E_{\nu}[
    \omega(X)-\mathcal B_\pi^\star\omega(X)
    \mid \omega(X)
  ].
\]
The tower property gives
\[
  \mathbb E_{\nu}[
    \{\omega(X)-\mathcal B_\pi^\star\omega(X)\}g\{\omega(X)\}
  ]
  =
  \mathbb E_{\nu}[
    r_\omega(X)g\{\omega(X)\}
  ].
\]
Taking the supremum over the unit ball of
square-integrable functions of \(\omega(X)\) gives
\(\|r_\omega\|_{L^2(\nu)}\), which is
\(\mathrm{Cal}_2(\omega)\) by
\eqref{eq:l2-calibration-refinement-errors}. This proves the dual
representation.
\end{proof}

\begin{lemma}[Exact empirical Bellman balance]
\label{lem:empirical-kkt-balance}
For every iteration \(k\ge0\) and every bounded measurable
\(g:\mathbb R\to\mathbb R\),
\begin{equation}
\label{eq:empirical-kkt-balance}
  P_n[
    \widehat\omega^{(k+1)}g\{\widehat\omega^{(k+1)}\}
  ]
  -
  \gamma Q_n[
    \widehat\omega^{(k)}(X)g\{\widehat\omega^{(k+1)}(X^+)\}
  ]
  -
  (1-\gamma)P_{0,m}[
    g\{\widehat\omega^{(k+1)}\}
  ]
  =0.
\end{equation}
\end{lemma}

\begin{proof}
Let \(\widehat u_j\) be the normalized PAVA solution at the fitted-value
design point \(t_j\), so that \(\sum_j a_j\widehat u_j=1\). Finiteness of the
objective implies \(b_j=0\) whenever \(\widehat u_j=0\). For bounded \(g\),
consider the support-preserving perturbation
\(u_{\varepsilon,j}:=\widehat u_j
\exp\{\varepsilon g(\widehat u_j)\}\).
Zero blocks remain zero, tied positive values receive the same perturbation,
and the finitely many distinct positive fitted block values remain ordered for
all sufficiently small positive and negative \(\varepsilon\). Differentiating the
finite-dimensional objective along this two-sided feasible path gives
\[
  0
  =
  \sum_{j=1}^J a_j\widehat u_jg(\widehat u_j)
  -
  \sum_{j=1}^J b_jg(\widehat u_j).
\]
Substituting the definitions of \(a_j\) and \(b_j\), together with the
fitted-value cell convention in Appendix~\ref{app:pava}, yields
\eqref{eq:empirical-kkt-balance}.
\end{proof}

\begin{lemma}[Bounded-density condition for subexponential smoothing]
\label{lem:calibration-bounded-kernel-smoothing}
Suppose
\[
  \frac{d\mu_{0,\pi}}{d\nu}\le L_0
  \qquad \nu\text{-almost surely},
\]
and suppose that \(P_\pi(\cdot\mid x)\) admits a jointly measurable density
\(p_\pi(\cdot\mid x)\) relative to \(\nu\) satisfying
\[
  \operatorname*{ess\,sup}_{(x,y)\sim\nu\otimes\nu}
  p_\pi(y\mid x)\le L_P.
\]
Then Condition~\ref{cond:iso-coverage} holds with
\[
  K_0\le \frac{L_0}{\log 2},
  \qquad
  K_+\le \frac{2L_P}{\log 2}.
\]
The same constants verify the enlarged smoothing requirement over
\(\mathcal W_N\cup\mathcal W_{{\rm b},N}\) in
Theorem~\ref{thm:isotonic-kl-regret}.
\end{lemma}

\begin{proof}
For any nonnegative \(\omega\) with
\(E_{\nu}\omega\le2\), Fubini's theorem gives
\[
  \frac{d\{(\omega\nu)P_\pi\}}{d\nu}(y)
  =
  \int \omega(x)p_\pi(y\mid x)\,d\nu(x)
  \le
  2L_P
\]
for \(\nu\)-almost every \(y\). Finally, if \(0\le Z\le L\), then
\(E\exp\{Z\log(2)/L\}\le2\), so
\(\|Z\|_{\psi_1}\le L/\log 2\). Applying this observation to the initial
density and the one-step density proves the bounds in
Lemma~\ref{lem:calibration-bounded-kernel-smoothing}. Every element of
\(\mathcal W_{{\rm b},N}\) has \(\nu\)-mass one. Applying the one-step
density bound with this normalization also verifies the enlarged smoothing
requirement.
\end{proof}

\begin{lemma}[Subexponential tails of the adjoint Bellman image and norm transfer]
\label{lem:coverage-norm-transfer}
Suppose Condition~\ref{cond:iso-coverage} holds. Let
\(\omega\in\mathcal W_N\), and put
\(Y_\omega:=\mathcal B_\pi^\star\omega\) and
\(K_{\rm sm}:=(1-\gamma)K_0+\gamma K_+\).
Then
\[
  \|Y_\omega\|_{\psi_1}\le K_{\rm sm},
\]
and, for every \(T\ge0\),
\[
  E_{\nu}(Y_\omega-T)_+
  \le2K_{\rm sm}e^{-T/K_{\rm sm}},
  \qquad
  \|(Y_\omega-T)_+\|_{L^2(\nu)}
  \le2K_{\rm sm}e^{-T/(2K_{\rm sm})}.
\]
If \(r\) is bounded, then, for every \(t\ge0\),
\begin{equation}
\label{eq:coverage-norm-transfer-bounds}
\begin{aligned}
  \|r\|_{L^2(\mu_{0,\pi})}^2
  &\le t\|r\|_{L^2(\nu)}^2
      +2K_0\|r\|_\infty^2e^{-t/K_0},\\
  Q_{\nu,\pi}\{\omega(X)^2r(X^+)^2\}
  &\le
  M_N\left\{
    t\|r\|_{L^2(\nu)}^2
    +2K_+\|r\|_\infty^2e^{-t/K_+}
  \right\}.
\end{aligned}
\end{equation}
Replacing \(K_0\) by \(K_+\) in the first inequality of
\eqref{eq:coverage-norm-transfer-bounds} also bounds
\(Q_{\nu,\pi}\{r(X^+)^2\}\).
\end{lemma}

\begin{proof}
The Bellman image is
\[
  Y_\omega
  =
  (1-\gamma)\frac{d\mu_{0,\pi}}{d\nu}
  +
  \gamma\frac{d\{(\omega\nu)P_\pi\}}{d\nu}.
\]
The triangle inequality for the Orlicz norm gives
\(\|Y_\omega\|_{\psi_1}\le K_{\rm sm}\). If
\(\|Z\|_{\psi_1}\le K\), Markov's inequality gives
\(P(Z>s)\le2e^{-s/K}\). Tonelli's theorem therefore gives
\begin{equation}
\label{eq:subexponential-tail-integrals}
\begin{aligned}
  E(Z-T)_+
  &=\int_T^\infty P(Z>s)\,ds
    \le2Ke^{-T/K},\\
  E(Z-T)_+^2
  &=2\int_T^\infty(s-T)P(Z>s)\,ds
    \le4K^2e^{-T/K}.
\end{aligned}
\end{equation}
Applying \eqref{eq:subexponential-tail-integrals} with \(Z=Y_\omega\) proves
the two tail inequalities in Lemma~\ref{lem:coverage-norm-transfer}.

For any nonnegative density \(s\),
\begin{equation}
\label{eq:density-norm-transfer}
  \int r^2s\,d\nu
  \le
  t\|r\|_{L^2(\nu)}^2
  +\|r\|_\infty^2E_{\nu}(s-t)_+.
\end{equation}
Apply \eqref{eq:density-norm-transfer} to the initial density and to the
successor density generated by \(\omega\), using
\(\omega^2\le M_N\omega\). Taking
\(\omega\equiv1\), which belongs to every \(\mathcal W_N\), proves the final
norm-transfer bound.
\end{proof}

For \(V<\infty\), define the bounded-variation fitted-value class
\[
  \mathcal V_{\widehat w_\pi}(V)
  =
  \left\{
    x\mapsto v\{\widehat w_\pi(x)\}:
    \|v\|_\infty+\operatorname{TV}(v)\le V
  \right\}.
\]

\begin{lemma}[Bounded variation of the upper-truncated calibration residual]
\label{lem:calibration-residual-bv}
Assume Condition~\ref{cond:bv-residual}. Put
\(\omega=\widehat\omega^{(K)}\), \(T=M_N\), and
\(q_T(t):=\mathbb E_{\nu}[
\{\mathcal B_\pi^\star\omega\}(X)\mid \omega(X)=t]\wedge T\).
Then both \(\omega\) and
\(e_T=\omega-q_T(\omega)\) belong to
\(\mathcal V_{\widehat w_\pi}(V_T)\), where one may take
\(V_T=4T+V_{\rm cal}\).
\end{lemma}

\begin{proof}
Write \(\omega=\theta\circ\widehat w_\pi\), where \(\theta\) is
nondecreasing and takes values in \([0,T]\). Hence
\[
\begin{aligned}
  \|\theta\|_\infty+\operatorname{TV}(\theta)
  &\le2T,\\
  \|q_T\circ\theta\|_\infty
  +\operatorname{TV}(q_T\circ\theta)
  &\le T+(V_{\rm cal}\vee T)
  \le2T+V_{\rm cal},\\
  \|\theta-q_T\circ\theta\|_\infty
  +\operatorname{TV}(\theta-q_T\circ\theta)
  &\le4T+V_{\rm cal}.
\end{aligned}
\]
Here the second line uses Condition~\ref{cond:bv-residual} and the fact that
composition with a nondecreasing map cannot increase total variation; the
third uses the triangle inequality for total variation.
Thus, both functions belong to
\(\mathcal V_{\widehat w_\pi}(V_T)\) with \(V_T=4T+V_{\rm cal}\).
\end{proof}

\begin{lemma}[Entropy of bounded-variation calibration classes]
\label{lem:bv-calibration-entropy}
Fix \(V<\infty\). There is a constant \(C_V<\infty\) such that, uniformly
over probability measures \(Q\),
\[
  \log N\{\epsilon,\mathcal V_{\widehat w_\pi}(V),L^2(Q)\}
  \le
  \frac{C_V}{\epsilon},
  \qquad 0<\epsilon\le 1 .
\]
The same bound, with a possibly larger \(C_V\), holds for the classes
\[
\begin{gathered}
  \{ab:a,b\in\mathcal V_{\widehat w_\pi}(V)\},\qquad
  \{(x,x^+)\mapsto a(x)b(x^+):
    a,b\in\mathcal V_{\widehat w_\pi}(V)\},\\
  \{(x,x^+)\mapsto b(x^+)^2:
    b\in\mathcal V_{\widehat w_\pi}(V)\}.
\end{gathered}
\]
For all four classes, the entropy constant and the uniform envelope may be
chosen no larger than \(C(1+V)^2\), for a universal \(C\).
\end{lemma}

\begin{proof}
Let \(\mathcal{BV}_V\) be the class of real functions \(v\) with
\(\|v\|_\infty+\operatorname{TV}(v)\le V\). By the Jordan decomposition, each
element of \(\mathcal{BV}_V\) is the difference of two nondecreasing functions
whose ranges are bounded by a constant depending only on \(V\). The monotone
bracketing bound used in Lemma~\ref{lem:isotonic-entropy} therefore gives
\[
  \sup_{\widetilde Q}
  \log N_{[]}\{\epsilon,\mathcal{BV}_V,L^2(\widetilde Q)\}
  \le
  \frac{C(1+V)}{\epsilon}.
\]
Composition with \(\widehat w_\pi\) transfers brackets under \(Q\) to brackets
under the pushforward distribution of \(\widehat w_\pi(X)\) when \(X\sim Q\).
Because bracketing numbers dominate covering numbers, this proves the entropy
bound for \(\mathcal V_{\widehat w_\pi}(V)\).

If \(a(x)=v_a\{\widehat w_\pi(x)\}\) and
\(b(x)=v_b\{\widehat w_\pi(x)\}\) with
\(v_a,v_b\in\mathcal{BV}_V\), then
\(\|v_av_b\|_\infty\le V^2\) and
\[
  \operatorname{TV}(v_av_b)
  \le
  \|v_a\|_\infty\operatorname{TV}(v_b)
  +
  \|v_b\|_\infty\operatorname{TV}(v_a)
  \le
  2V^2 .
\]
Thus the pointwise product class is contained in a bounded-variation class
with radius at most \(3V^2\), so the entropy bound for
\(\mathcal V_{\widehat w_\pi}(V)\) applies.
Taking \(v_a=v_b\) gives the bound for \(b^2\).

For the two-coordinate product class, let \(Q_1\) and \(Q_2\) be the two
marginals of \(Q\). If
\(\|a-a'\|_{L^2(Q_1)}\le\epsilon\) and
\(\|b-b'\|_{L^2(Q_2)}\le\epsilon\), then
\[
  \|a(X)b(X^+)-a'(X)b'(X^+)\|_{L^2(Q)}
  \le
  C_V\epsilon .
\]
Covering the two marginal classes at radius \(\epsilon/(2V)\) and multiplying
the covering numbers gives the entropy bound for the product class. The same
calculations bound the entropy constants and envelopes by \(C(1+V)^2\), as
claimed.
\end{proof}

\begin{lemma}[Local maximal inequality for \(1/\epsilon\)-entropy classes]
\label{lem:localized-maximal-bv}
Let \(\mathcal F\) be a uniformly bounded class satisfying
\[
  \sup_Q\log N\{\epsilon,\mathcal F,L^2(Q)\}
  \le
  \frac{A}{\epsilon},
  \qquad 0<\epsilon\le1 .
\]
If \(Z_1,\ldots,Z_N\sim P\) independently and
\(\rho_N=N^{-1/3}+\sqrt{\log(1/\delta)/N}\), then, with probability at least
\(1-\delta\),
\[
  |(P_N-P)f|
  \le
  C_{A,M}
  \left\{
    \rho_N\|f\|_{L^2(P)}
    +
    \rho_N^2
  \right\}
  \qquad\text{for all } f\in\mathcal F ,
\]
where \(M\) is a uniform envelope for \(\mathcal F\).
One may take \(C_{A,M}\le C(1+A+M)\), with \(C\) universal.
\end{lemma}

\begin{proof}
Let \(u=\log(1/\delta)\) and
\(\mathcal F(r)=\{f\in\mathcal F:\|f\|_{L^2(P)}\le r\}\). The entropy
assumption and Lemma~\ref{lem:localized-entropy-bound} imply
\begin{equation}
\label{eq:localized-rademacher-rate}
  \mathcal R_N(\mathcal F,r;P)
  \le
  C(1+A+M)
  \left\{
    \sqrt{\frac{r}{N}}
    +
    \frac{1}{Nr}
  \right\}.
\end{equation}
For every radius \(r\ge\rho_N\ge N^{-1/3}\), the term \(1/(Nr)\) in
\eqref{eq:localized-rademacher-rate} is bounded by \(\sqrt{r/N}\).
Symmetrization therefore gives
\[
  \mathbb E
  \sup_{f\in\mathcal F(r)}
  |(P_N-P)f|
  \le
  2\mathcal R_N(\mathcal F,r;P)
  \le
  C_{A,M}\sqrt{\frac{r}{N}} .
\]
Apply Lemma~\ref{lem:tool-bousquet} to the centered class
\(\{f-Pf:f\in\mathcal F(r)\}\). The centered envelope is at most \(2M\), and
the variance is at most \(r^2\). Hence, with probability at least \(1-e^{-v}\),
\begin{equation}
\label{eq:localized-bousquet-bound}
  \sup_{f\in\mathcal F(r)}
  |(P_N-P)f|
  \le
  C_{A,M}
  \left\{
    \sqrt{\frac{r}{N}}
    +
    r\sqrt{\frac{v}{N}}
    +
    \frac{v}{N}
  \right\}.
\end{equation}
We apply \eqref{eq:localized-bousquet-bound} on the dyadic shells
\(\{f:2^{j-1}\rho_N<\|f\|_{L^2(P)}\le2^j\rho_N\}\), together with the ball
\(\{f:\|f\|_{L^2(P)}\le\rho_N\}\). Use
\(v_j=u+2\log(j+1)+\log(\pi^2/6)\) and take a union bound over \(j\). Since
\(\rho_N\ge N^{-1/3}\) and \(\rho_N^2N\ge u+N^{1/3}\), the Bousquet bound is
at most
\[
  C_{A,M}\{\rho_N r_j+\rho_N^2\},
  \qquad r_j=2^j\rho_N,
\]
on the \(j\)th shell and at most \(C_{A,M}\rho_N^2\) on the innermost ball.
Replacing \(r_j\) by \(\|f\|_{L^2(P)}\) and enlarging the numerical constant
gives the asserted local maximal inequality and the stated dependence on
\(A,M\). The four classes in
Lemma~\ref{lem:bv-calibration-empirical-process} have finite \(L^2(Q)\)
covering numbers at every positive radius. The local maximal inequality
therefore extends from a countable dense subclass to each full class by taking
limits over the net radius.
\end{proof}

\begin{lemma}[Localized empirical-process bound for bounded-variation calibration classes]
\label{lem:bv-calibration-empirical-process}
Fix \(V<\infty\) and condition on the training data used to fit
\(\widehat w_\pi\). Let
\[
  \rho_{n,m}
  =
  (n\wedge m)^{-1/3}
  +
  \sqrt{\frac{\log(1/\delta)}{n\wedge m}} .
\]
With conditional probability at
least \(1-\delta\), the following bounds hold simultaneously for
\(P_n\), \(Q_n\), and \(P_{0,m}\). For all
\(a,b\in\mathcal V_{\widehat w_\pi}(V)\),
\begin{equation}
\label{eq:bv-empirical-process-bounds}
\begin{aligned}
  |(P_n-\nu)(ab)|
  &\le
  C_V\{\rho_{n,m}\|ab\|_{L^2(\nu)}+\rho_{n,m}^2\},\\
  |(P_{0,m}-\mu_{0,\pi})a|
  &\le
  C_V\{\rho_{n,m}\|a\|_{L^2(\mu_{0,\pi})}+\rho_{n,m}^2\},\\
  |(Q_n-Q_{\nu,\pi})\{a(X)b(X^+)\}|
  &\le
  C_V\{\rho_{n,m}\|a(X)b(X^+)\|_{L^2(Q_{\nu,\pi})}+\rho_{n,m}^2\},\\
  |(Q_n-Q_{\nu,\pi})\{b(X^+)^2\}|
  &\le
  C_V\{\rho_{n,m}\|b(X^+)^2\|_{L^2(Q_{\nu,\pi})}+\rho_{n,m}^2\},
\end{aligned}
\end{equation}
where \(Q_{\nu,\pi}\) is the law of \((X,X^+)\) under
\(X\sim\nu\) and \(X^+\mid X\sim P_\pi(\cdot\mid X)\).
The constant in all four inequalities in
\eqref{eq:bv-empirical-process-bounds} may be chosen no larger than
\(C(1+V)^2\).
\end{lemma}

\begin{proof}
Lemma~\ref{lem:bv-calibration-entropy} verifies the entropy condition in
Lemma~\ref{lem:localized-maximal-bv} for each of the four classes appearing in
the statement. Applying Lemma~\ref{lem:localized-maximal-bv} to the
corresponding samples, with failure probability \(\delta/4\) for each class,
and taking a union bound gives simultaneous control. Since
\(n\wedge m\le n,m\), replacing either sample size by \(n\wedge m\) only
enlarges the upper bounds. This proves the lemma.
\end{proof}

\begin{lemma}[Truncated calibration comparison]
\label{lem:truncated-calibration-bridge}
Put
\[
  \omega=\widehat\omega^{(K)},\qquad
  \omega^-=\widehat\omega^{(K-1)},\qquad
  T=M_N,
\]
and define
\[
\begin{aligned}
  Y&=\mathcal B_\pi^\star\omega,&
  \Gamma_\omega(t)&=E_{\nu}\{Y(X)\mid\omega(X)=t\},\\
  q_T(t)&=\Gamma_\omega(t)\wedge T,&
  e_T&=\omega-q_T(\omega),\\
  \eta_2(T)&=\|(Y-T)_+\|_{L^2(\nu)}.&
\end{aligned}
\]
Then
\begin{equation}
\label{eq:truncated-calibration-error-bound}
  \mathrm{Cal}_2(\omega)
  \le
  \|e_T\|_{L^2(\nu)}+\eta_2(T),
\end{equation}
and
\begin{equation}
\label{eq:truncated-kkt-comparison}
\begin{aligned}
  \|e_T\|_{L^2(\nu)}^2
  \le{}&
  |(\nu-P_n)(\omega e_T)|
  +(1-\gamma)|(\mu_{0,\pi}-P_{0,m})e_T|\\
  &+\gamma\left|
    (Q_{\nu,\pi}-Q_n)\{\omega(X)e_T(X^+)\}
  \right|\\
  &+\gamma\left|
    Q_n\{(\omega-\omega^-)(X)e_T(X^+)\}
  \right|.
\end{aligned}
\end{equation}
\end{lemma}

\begin{proof}
Conditional Jensen's inequality gives
\[
  \|(\Gamma_\omega-T)_+\circ\omega\|_2
  \le\|(Y-T)_+\|_2=\eta_2(T).
\]
Since
\(\omega-\Gamma_\omega(\omega)
=e_T-(\Gamma_\omega-T)_+\circ\omega\), the triangle inequality proves the
calibration-error bound \eqref{eq:truncated-calibration-error-bound}.

To prove \eqref{eq:truncated-kkt-comparison}, observe that
\begin{equation}
\label{eq:truncated-population-balance-lower-bound}
\begin{aligned}
  E_{\nu}[\{\omega-\Gamma_\omega(\omega)\}e_T]
  &=
  \|e_T\|_2^2
  -E_{\nu}[\{(\Gamma_\omega-T)_+\circ\omega\}e_T]\\
  &\ge\|e_T\|_2^2,
\end{aligned}
\end{equation}
because \(e_T=\omega-T\le0\) wherever
\(\Gamma_\omega(\omega)>T\). The left-hand side of
\eqref{eq:truncated-population-balance-lower-bound} is the population
Bellman-balance functional evaluated at \(e_T\).
Lemma~\ref{lem:empirical-kkt-balance}, applied with
\(g(t)=t-q_T(t)\), shows that the corresponding empirical functional is zero.
Adding and subtracting this empirical functional yields
\eqref{eq:truncated-kkt-comparison}, which proves the lemma.
\end{proof}

\begin{proof}[Proof of Theorem~\ref{thm:isotonic-calibration-error}]
Abbreviate
\[
\begin{aligned}
  T&=M_N,
  &\omega&=\widehat\omega^{(K)},
  &\omega^-&=\widehat\omega^{(K-1)},\\
  e&=e_T=\omega-q_T(\omega),
  &a&=\|e\|_{L^2(\nu)},
  &\Delta_K&=\|\omega-\omega^-\|_{n,2},\\
  \rho_N
  &=N^{-1/3}+\sqrt{\frac{\log(1/\delta)}{N}},
  &L_N&=\log(eN),
  &V_N&=4T+V_{\rm cal}+2.
\end{aligned}
\]
Lemma~\ref{lem:calibration-residual-bv} gives
\(e,\omega\in\mathcal V_{\widehat w_\pi}(V_N)\), and the same class contains
\(\omega^-\). We work on the event in
Lemma~\ref{lem:bv-calibration-empirical-process} with \(V=V_N\).

We first verify that the realized iterates belong to \(\mathcal W_N\), which
permits the use of Condition~\ref{cond:iso-coverage}. Since
\(P_n\omega=P_n\omega^-=1\), the source-sample product bound in
\eqref{eq:bv-empirical-process-bounds} gives
\begin{equation}
\label{eq:calibration-iterate-mass-control}
  |E_{\nu}\omega-1|
  \vee
  |E_{\nu}\omega^--1|
  \le
  C(1+V_N)^2\{T\rho_N+\rho_N^2\}.
\end{equation}
The upper bound in \eqref{eq:calibration-iterate-mass-control} is \(o(1)\)
because \(T=O(\log N)\). Hence, for all
sufficiently large \(N\), both population masses lie in \([1/2,2]\), so
\(\omega,\omega^-\in\mathcal W_N\).

Take \(t_N=(1+K_0+K_+)L_N\) in
Lemma~\ref{lem:coverage-norm-transfer}. Since \(\|e\|_\infty\le T\), its
transfer bounds give
\begin{equation}
\label{eq:calibration-transfer-bounds}
\begin{aligned}
  \|e\|_{L^2(\mu_{0,\pi})}
  &\le C_1\sqrt{L_N}(a+\rho_N),\\
  \|\omega(X)e(X^+)\|_{L^2(Q_{\nu,\pi})}
  &\le C_1\sqrt{L_N}(a+\rho_N),
\end{aligned}
\end{equation}
where one may take
\(C_1\le C(1+K_0+K_++T)^2\). The exponential remainders in
\eqref{eq:calibration-transfer-bounds} are bounded by
\(CT^{3/2}(eN)^{-1/2}\), and hence by
\(C\rho_N\), for all sufficiently large \(N\). Since
\(\|\omega e\|_2\le Ta\),
Lemma~\ref{lem:bv-calibration-empirical-process} implies that the first three
empirical-process terms in \eqref{eq:truncated-kkt-comparison} are bounded by
\begin{equation}
\label{eq:calibration-empirical-process-terms}
  C_2\sqrt{L_N}\{\rho_Na+\rho_N^2\},
\end{equation}
with
\(C_2\le C(1+K_0+K_++V_{\rm cal}+T)^5\).

For the iteration term, the unweighted successor transfer gives
\begin{equation}
\label{eq:successor-square-transfer}
  Q_{\nu,\pi}\{e(X^+)^2\}
  \le
  C_1L_N(a+\rho_N)^2.
\end{equation}
Moreover,
\begin{equation}
\label{eq:successor-square-fourth-moment}
  \|e(X^+)^2\|_{L^2(Q_{\nu,\pi})}
  \le
  T\|e(X^+)\|_{L^2(Q_{\nu,\pi})}.
\end{equation}
Applying the successor-square bound in
\eqref{eq:bv-empirical-process-bounds}, together with
\eqref{eq:successor-square-transfer} and
\eqref{eq:successor-square-fourth-moment}, yields
\[
\begin{aligned}
  Q_n\{e(X^+)^2\}
  \le{}&
  C_1L_N(a+\rho_N)^2\\
  &+C(1+V_N)^2
  \left\{
    \rho_NT\sqrt{C_1L_N}(a+\rho_N)+\rho_N^2
  \right\}
\end{aligned}
\]
and therefore
\begin{equation}
\label{eq:successor-square-empirical-bound}
  Q_n\{e(X^+)^2\}
  \le
  C_3L_N(a+\rho_N)^2,
\end{equation}
where
\(C_3\le C(1+K_0+K_++V_{\rm cal}+T)^5\). By
\eqref{eq:successor-square-empirical-bound}, Cauchy--Schwarz on the common
transition sample gives
\begin{equation}
\label{eq:iteration-transition-product-bound}
  \left|
    Q_n\{(\omega-\omega^-)(X)e(X^+)\}
  \right|
  \le
  \sqrt{C_3L_N}\,\Delta_K(a+\rho_N).
\end{equation}

Combining \eqref{eq:calibration-empirical-process-terms} and
\eqref{eq:iteration-transition-product-bound} with
\eqref{eq:truncated-kkt-comparison} gives
\begin{equation}
\label{eq:calibration-quadratic-inequality}
  a^2
  \le
  \kappa_{0,N}\sqrt{L_N}
  (\rho_N+\Delta_K)(a+\rho_N),
\end{equation}
where \(\kappa_{0,N}
\le C\{1+K_0+K_++V_{\rm cal}+T\}^{6}\).
To track this constant, set
\(H_N=1+K_0+K_++V_{\rm cal}+T\). The bounded-variation empirical-process
constant is at most \(CH_N^2\). The initial-law and weighted-successor
transfers contribute, respectively,
\(C\sqrt{L_N}(a+\rho_N)\) and
\(C\sqrt{TL_N}(a+\rho_N)\). The factor \(\sqrt T\), together with
\(\|\omega e\|_2\le Ta\), is absorbed into the polynomial in \(H_N\).
Consequently, the first three terms in
\eqref{eq:truncated-kkt-comparison} are bounded by
\(CH_N^5\sqrt{L_N}\{\rho_Na+\rho_N^2\}\), and its iteration term contributes
no larger factor. Thus, \(CH_N^6\) bounds every coefficient in the quadratic
inequality \eqref{eq:calibration-quadratic-inequality}.

Since \(\rho_N\le\rho_N+\Delta_K\), solving
\eqref{eq:calibration-quadratic-inequality} gives
\begin{equation}
\label{eq:calibration-pre-absorption}
  a
  \le
  \kappa_{{\rm cal},N}\sqrt{L_N}(\rho_N+\Delta_K),
\end{equation}
where \(\kappa_{{\rm cal},N}
\le C\{1+K_0+K_++V_{\rm cal}+T\}^{8}\).
Lemma~\ref{lem:truncated-calibration-bridge} now yields the explicit
bound before tail absorption:
\[
  \mathrm{Cal}_2(\omega)
  \le
  \kappa_{{\rm cal},N}\sqrt{L_N}(\rho_N+\Delta_K)
  +\eta_2(T).
\]

Because \(\omega\in\mathcal W_N\),
Lemma~\ref{lem:coverage-norm-transfer} gives
\begin{equation}
\label{eq:calibration-tail-remainder}
  \eta_2(T)
  \le2K_{\rm sm}e^{-T/(2K_{\rm sm})}.
\end{equation}
Since \(T=1\vee A_{\rm env}\log(eN)\) and
\(A_{\rm env}>2K_{\rm sm}/3\), the upper bound in
\eqref{eq:calibration-tail-remainder} is \(o(N^{-1/3})\), even after
multiplication by any fixed polynomial in \(T\). For all sufficiently large
\(N\), this remainder is therefore absorbed by
\(\kappa_{{\rm cal},N}\sqrt{L_N}\rho_N\). Substituting the definitions of
\(\rho_N\), \(L_N\), and \(\Delta_K\) gives the bound in
Theorem~\ref{thm:isotonic-calibration-error}.
\end{proof}

 \end{document}